\documentclass[11pt]{article}
\usepackage[a4paper]{geometry}

\usepackage{xspace}
\usepackage{amsmath, amsthm, amssymb, dsfont, mdwlist}
\usepackage{citesort}
\usepackage{enumitem}

\allowdisplaybreaks[4]

\newtheorem{theorem}{Theorem}

\newtheorem{lemma}[theorem]{Lemma}

\newcommand{\ignore}[1]{}

\newcommand{\wlo}{w.\,l.\,o.\,g.\xspace}

\newcommand{\ie}{i.\,e.\xspace}
\newcommand{\eg}{e.\,g.\xspace}

\newcommand{\ONEMAX}{\textsc{OneMax}\xspace}
\newcommand{\OneMax}{\ONEMAX}
\newcommand{\onemax}{\ONEMAX}
\newcommand{\cga}{cGA\xspace}

\newcommand{\indic}[1]{\mathds{1}\{#1\}}
\DeclareMathOperator{\sgn}{sgn}

\newcommand{\oneoneea}{(1+1)~EA\xspace}

\renewcommand{\phi}{\varphi}

\usepackage{tikz}
\usepgflibrary{arrows}
\usetikzlibrary{arrows}
\usepackage{pgfplots}
\usepackage{pgfplotstable}
\tikzset{every mark/.append style={scale=0.60}}

\usepackage{graphicx}
\usepackage{xargs}
\usepackage{xcolor}
\usepackage[algo2e,ruled,vlined]{algorithm2e}
\SetAlgoSkip{}
\DontPrintSemicolon

\newcommand{\prob}[1]{\mathord{\ensuremath{\mathrm{P}}}\mathord{\left[#1\right]}}         % Pr with argument.

\newcommand{\lANT}[1][$\lambda$]{#1\nobreakdash-MMAS$_{\textrm{ib}}$\xspace}
\newcommand{\twoant}{\lANT[2]}

\newcommand{\E}[1]{\mathrm{E}\mathord{\left(#1\right)}}
\newcommand{\expect}[1]{\E{#1}}
\DeclareMathOperator{\Var}{Var}
\newcommand{\R}{\mathds{R}}
\newcommand{\N}{\mathds{N}}

\newcommand{\card}[1]{|#1|}
\newcommand{\poly}{\mathrm{poly}}

\newcommand{\pot}{\varphi}

\renewcommand{\epsilon}{\varepsilon}

\renewenvironment{proof}%
{\begin{trivlist}\item\textbf{Proof.}}%
{\hspace*{\fill}$\Box$\end{trivlist}}
\newenvironment{proofof}[1]%{\begin{proof}[of~#1]}{\qed\end{proof}}
{\begin{trivlist}\item\textbf{Proof of #1.}}%
{\hspace*{\fill}$\Box$\end{trivlist}}

\begin{document}
\title{Update Strength in EDAs and ACO: How to Avoid Genetic~Drift}

%\numberofauthors{2}
\author{
	 Dirk Sudholt\\
	Department of Computer Science \\
	University of Sheffield \\
	Sheffield, United Kingdom \\
	\and
	 Carsten Witt\\
	DTU Compute\\
  Technical University of Denmark\\
	Kongens Lyngby, Denmark%\\
	%\email{cawi@dtu.dk}\\
}

\maketitle

\begin{abstract}
We provide a rigorous runtime analysis concerning the update strength, a vital parameter in probabilistic model-build\-ing GAs such as the step size~$1/K$ in the compact Genetic Algorithm (cGA) and the evaporation factor~$\rho$ in ACO. While a large update strength is desirable for exploitation, there is a general trade-off: too strong updates can lead to genetic drift and poor performance. We demonstrate this trade-off for the cGA and a simple MMAS ACO algorithm on the OneMax function. More precisely, we obtain lower bounds on the expected runtime of $\Omega(K\sqrt{n} + n \log n)$ and $\Omega(\sqrt{n}/\rho + n \log n)$, respectively, showing that the update strength should be limited to $1/K, \rho = O(1/(\sqrt{n} \log n))$.
In fact, choosing $1/K, \rho \sim 1/(\sqrt{n}\log n)$ both algorithms efficiently optimize OneMax in expected time $O(n \log n)$.
Our analyses provide new insights into the stochastic behavior of probabilistic model-building GAs and propose new guidelines for setting the update strength in global optimization.
\end{abstract}

\section{Introduction}

The term \emph{probabilistic model-building GA} describes a class of algorithms that construct a probabilistic model which is used to generate new search points. The model is adapted using information about previous search points. Both es\-ti\-ma\-tion-of-distribution algorithms (EDAs) and swarm intelligence algorithms including ant colony optimizers (ACO) and particle swarm optimizers (PSO) fall into this class.
These algorithms generally behave differently from evolutionary algorithms where a population of search points fully describes the current state of the algorithm.

EDAs like the compact Genetic Algorithm (\cga) and many ACO algorithms update their probabilistic models by sampling new solutions and then updating the model according to information about good solutions found. In this work we focus on binary search spaces and simple univariate probabilistic models, that is, for each bit there is a value $p_i$ that determines the probability of setting the $i$-th bit to~1 in a newly created solution.

The compact Genetic Algorithm was introduced by Harik, Lobo and Goldberg~\cite{HarikEtAlCGA}. 
In brief, simulates the behavior of a Genetic Algorithm with population size~$K$ in a more compact fashion.
In each iteration two solutions are generated, and if they differ in fitness, $p_i$ is updated by $\pm 1/K$ in the direction of the fitter individual. Here $1/K$ reflects the strength of the update of the probabilistic model.
Simple ACO algorithms based on the Max-Min Ant System (MMAS)~\cite{Stutzle2000}, using the iteration-best update rule, behave similarly: they generate a number $\lambda$ of solutions and reinforce the best solution amongst these by increasing values $p_i$, here called \emph{pheromones}, according to $(1-\rho) p_i + \rho$ if the best solution had bit~$i$ set to~$1$, and $(1-\rho)p_i$ otherwise.
Here the parameter $0 < \rho < 1$ is called \emph{evaporation factor}; it plays a similar role to the update strength $1/K$ for \cga.

Neumann, Sudholt, and Witt~\cite{Neumann2010a} showed that $\lambda=2$ ants suffice to optimize the function OneMax$(x) := \sum_{i=1}^n x_i$, a simple hill-climbing task, in expected time $O(n \log n)$ if the update strength is chosen small enough, $\rho \le 1/(c\sqrt{n}\log n)$ for a suitably large constant~$c > 0$. If $\rho$ is chosen unreasonably large, $\rho \ge c'/(\log n)$ for some $c'>0$, the algorithm shows a chaotic behavior and needs exponential time even on this very simple function. In a more general sense, this result suggests that for global optimization such high update strengths should be avoided for any problem, unless the problem contains many global optima.

However, these results leave open a wide gap of parameter values between $\sim 1/(\log n)$ and $\sim 1/(\sqrt{n}\log n)$, for which no results are available. This leaves open the question of which update strengths are optimal, and for which values
performance degrades. Understanding the working principles of the underlying probabilistic model remains an important open problem for both \cga and ACO algorithms. This is evident from the lack of reasonable lower bounds. To date, the best known direct
lower bound for MMAS algorithms for reasonable parameter choices is $\Omega((\log n)/\rho - \log n)$~\cite[Theorem~5]{Neumann2009}. The best known lower bound for \cga is $\Omega(K \sqrt{n})$~\cite{Droste2006a}.
There are more general bounds from black-box complexity theory \cite{DJWBlackBox,DoerrLenglerGECCO2015}, showing that the expected runtime of comparison-based algorithms such as MMAS
must be $\Omega(n)$ on \OneMax. However, these black-box bounds do not yield direct insight into the stochastic behavior of the algorithms and do not
shed light on the dependency of the algorithms' performance on the update strength.

In this paper, we study \twoant and \cga with a much more detailed analysis that provides such insights through rigorous runtime analysis. We prove lower bounds of\linebreak[4] $\Omega(K\sqrt{n} + n \log n)$ and $\Omega(1/\rho \cdot \sqrt{n} + n \log n)$. The terms $K \sqrt{n}$ and $1/\rho \cdot \sqrt{n}$ indicate that the runtime decreases when the update strength $1/K$ or $\rho$ is increased. However, the added terms $\mathrel{+}n \log n$ set a limit: there is no asymptotic decrease and hence no benefit for choosing update strengths $1/K$ or $\rho$ growing faster than $1/(\sqrt{n} \log n)$. The reason is that in this regime both algorithms suffer from genetic drift that leads to incorrect decisions being made. Correcting these incorrect decisions requires time $\Omega(n \log n)$. These lower bounds hold in expectation and with high probability; hence, they accurately reflect the algorithms' typical performance.

We further show that these bounds are tight for $1/K, \rho \le 1/(c\sqrt{n}\log n)$. In this parameter regime the impact of genetic drift is bounded and hence these parameter choices provably lead to the best asymptotic performance on OneMax for arbitrary problem sizes~$n$.

The lower bounds formally apply to OneMax, but can be regarded as general limitations for global optimization on functions with a small number of optima.
%Since OneMax is probably the easiest function with a unique global optimum for the considered algorithms (see~\cite{Droste2006a} for a related result for \cga and~\cite{Doerr2010} for a formal proof for \oneoneea), our results suggest that for global optimization, update strengths $1/K$ and $\rho$ growing faster than $1/(\sqrt{n}\log n)$ should be avoided for any problem, unless the problem contains many global optima.
Among all functions with a unique global optimum, the function OneMax is provably the easiest function for certain evolutionary algorithms (see~\cite{Doerr2010} for a proof for the \oneoneea and~\cite{Sudholt2012c,Witt2013} for extensions to populations), and similar results were shown for the \cga on linear functions by Droste~\cite{Droste2006a}. We believe that the lower bounds give general performance limits for all functions with a unique global optimum (however, new arguments will be required to show this formally).

From a technical point of view, our work uses a novel approach: using a second-order potential function to approximate the distribution of hitting times for a random walk that underlies changes in the probabilistic model.
We are confident that this approach will find application in other stochastic processes.

Finally, by pointing out similarities between \cga and \twoant, using the same analytical framework to understand changes in the probabilistic model, we make a step towards a unified theory of probabilistic model-building GAs.

%In a more general sense, these parameters allow for the \cga and \twoant to perform hill climbing while maximising

This report is structured as follows. Section~\ref{sec:preliminaries} introduces the algorithms and Section~\ref{sec:superposition} 
presents important analytical concepts. Section~\ref{sec:upper} 
proves efficient upper bounds for small update strengths, whereas Section~\ref{sec:lower} deals with the lower bounds for large update strengths. We 
finish with some conclusions.

\section{Preliminaries}
\label{sec:preliminaries}

Our presentation of \cga{} follows Droste~\cite{Droste2006a};
see also~Friedrich, K{\"o}tzing, Krejca, and Sutton~\cite{FriedrichEtAlISAAC15}.
The parameter $1/K$ is
called update strength (classically, $K$ is called population size) and the $p_{i,t}$ are called marginal probabilities.
Pseudocode of \cga is shown in Algorithm~\ref{alg:cGA}. The simple MMAS algorithm \twoant, analyzed before in \cite{Neumann2010a}\footnote{\small The \twoant
in \cite{Neumann2010a} used a randomized tie-breaking rule, which we replaced by a deterministic one here. This does not affect the stochastic behavior
on \OneMax but eases the analysis.},
is shown in Algorithm~\ref{alg:MMAS}. Note that
the two algorithms only differ in the update mechanism. In the context of ACO, $p_{i,t}$ are usually called pheromone values,
however we also refer to them as marginal probabilities to unify our approach to both algorithms.

We note that the marginal probabilities for both algorithms are restricted to the interval $[1/n,1-1/n]$.
These bounds are used such that the algorithms always show a finite expected optimization time, as otherwise certain bits can be irreversibly fixed to~0 or~1. Our results also apply to algorithms without these borders: our analysis can be easily adapted to show that when the optimum is found efficiently in the presence of borders, it is found with high probability when borders are removed, and when the algorithm is inefficient, many bits are fixed opposite to the optimum.

\begin{algorithm2e}[h]
  $t \gets 0$
  $p_{1,t} \gets p_{2,t} \gets \cdots \gets p_{n,t} \gets 1/2$
  \While{termination criterion not met}{
    \For{$i \in \{1,\dots,n\}$\label{li:x}}{
      $x_i \gets 1$ with prob.\ $p_{i,t}$, $x_i \gets 0$ with prob.\ $1 - p_{i,t}$
    }
  \For{$i \in \{1,\dots,n\}$\label{li:y}}{
    $y_i \gets 1$ with prob.\ $p_{i,t}$, $y_i \gets 0$ with prob.\ $1 - p_{i,t}$
  }
  \lIf{$f(x) < f(y)$\label{li:eval}}{swap $x$ and $y$}
  \For{$i \in \{1,\ldots,n\}$}{
    \lIf{$x_i > y_i$}{$p_{i,t+1} \gets p_{i,t} + 1/K$}
    \lIf{$x_i < y_i$}{$p_{i,t+1} \gets p_{i,t} - 1/K$}
    \lIf{$x_i = y_i$}{$p_{i,t+1} \gets p_{i,t}$}
		Restrict $p_{i,t+1}$ to be within $[1/n,1-1/n]$
  }
	$t \gets t+1$
}
\caption{Compact Genetic Algorithm (\cga)}
\label{alg:cGA}
\end{algorithm2e}

\begin{algorithm2e}[h]
  $t \gets 0$
  $p_{1,t} \gets p_{2,t} \gets \cdots \gets p_{n,t} \gets 1/2$
  \While{termination criterion not met}{
    \For{$i \in \{1,\dots,n\}$\label{li:cx}}{
      $x_i \gets 1$ with prob.\ $p_{i,t}$, $x_i \gets 0$ with prob.\ $1 - p_{i,t}$
    }
  \For{$i \in \{1,\dots,n\}$\label{li:cy}}{
    $y_i \gets 1$ with prob.\ $p_{i,t}$, $y_i \gets 0$ with prob.\ $1 - p_{i,t}$
  }
  \lIf{$f(x) < f(y)$\label{li:ceval}}{swap $x$ and $y$}
%  \lIf{$f(x) = f(y)$\label{li:eval}}{swap $x$ and $y$ with prob.~$1/2$}
  \For{$i \in \{1,\ldots,n\}$}{
    \lIf{$x_i \ge y_i$}{$p_{i,t+1} \gets (1-\rho)p_{i,t} + \rho$}
    \lIf{$x_i < y_i$}{$p_{i,t+1} \gets (1-\rho)p_{i,t}$}
		Restrict $p_{i,t+1}$ to be within $[1/n,1-1/n]$
  }
	$t \gets t+1$
}
\caption{\twoant}
\label{alg:MMAS}
\end{algorithm2e}

There are intriguing similarities in the definition of \cga and \twoant, despite these
two algorithms coming from quite different strands from the EC community. As said, they only
differ in the update mechanism: \cga uses a symmetrical update rule with $1/K$ as the amount of change and changes
a marginal probability if and only if both offspring differ in the corresponding bit value.
\twoant will always change a marginal probability in either positive or negative direction
by a value dependent on its current state; however, the maximum absolute
change will always be at most $\rho$. We are not the first to point out these similarities (\eg,
see the survey by Hauschild and Pelikan \cite{HauschildPelikan11}, who embrace both algorithms
under the umbrella of EDAs). However, our analyses will reveal the surprising insight that
both \cga and \twoant have the same runtime behavior as well as the same
optimal parameter set on \OneMax and can be analyzed with almost the same techniques.

\section{On the Dynamics of the Probabilistic Model}
\label{sec:superposition}

We first elaborate on the stochastic processes underlying the probabilistic model in both algorithms. These insights will then be used to prove upper runtime bounds for small update strengths in Section~\ref{sec:upper} and lower runtime bounds for large update strengths in Section~\ref{sec:lower}.

We fix an arbitrary bit~$i$
and $p_{i,t}$, its marginal probability at time~$t$. Note that $p_{i, t}$ is a random variable, and so is its random change $\Delta_{t}:=p_{i, t+1}-p_{i, t}$ in one step. This change depends on whether the value of bit~$i$ matters for the
decision whether to update with respect to the first bit string~$x$ sampled in iteration~$t$ (using
$p_{\cdot,t}$ as sampling distribution) or the second one~$y$ (cf.\ also \cite{Neumann2010a}).
More precisely, we inspect $D_t:=\card{x}-\card{x_{i}}-(\card{y}-\card{y_{i}})$, which is the change
of $\OneMax$-value at bits other than~$i$.

We assume $p_{i, t}$ to be bounded away from the borders such that $\Delta_t$ is not affected by the borders. Then we get for \cga:
\begin{itemize}
\item If $\card{D_t}\ge 2$, then bit~$i$ does not affect the decision whether to update
with respect to~$x$ or~$y$. For $\Delta_t>0$ it is necessary
that bit~$i$ is sampled differently. Hence, the $p_{i,t}$\nobreakdash-value
increases and decreases by $1/K$ with
equal probability $p_{i,t}(1-p_{i,t})$; with the remaining probability %$1-2p_{i,t}+2p_{i,t}^2$,
$p_{i,t+1}=p_{i,t}$. The change in this case is defined by $\Delta_t = F_t$ where
\[
F_t := \begin{cases}
+1/K & \text{ with probability $p_{i,t}(1-p_{i,t})$},\\
-1/K & \text{ with probability $p_{i,t}(1-p_{i,t})$},\\
0 & \text{ with the remaining probability}.
\end{cases}
\]
We call a step where $\card{D_t}\ge 2$ a \emph{random-walk step (rw-step)} since the process in such a step is a fair random walk (with self-loops) as ${\expect{\Delta_t \mid p_{i, t}} = \expect{F_t \mid p_{i, t}} = 0}$.

If $D_t = 1$ then $\card{x_{t+1}} \ge \card{y_{t+1}}$ such that $x_{t+1}$ and $y_{t+1}$ are never swapped in line 8 of
\cga. Hence, the same argumentation as in the previous case applies and the process
performs an rw-step as well.
\item
If $D_t = -1$ then $x_{t+1}$ and $y_{t+1}$ are swapped unless bit~$i$ is sampled to~$1$ in
$x_{t+1}$ and to~$0$ in~$y_{t+1}$. Hence, both events of sampling bit~$i$ differently
increase the $p_{i,t}$-value. We have $\Delta_t=1/K$ with probability $2p_{i,t}(1-p_{i,t})$
and $\Delta_t=0$ otherwise.

If $D_t=0$ then as in the case $D_t=-1$ both events of sampling bit~$i$ differently increase
the $p_{i,t}$-value. Hence, we again have $\Delta_t=1/K$ with probability $2p_{i,t}(1-p_{i,t})$
and $\Delta_t=0$ otherwise. Let $B_t$ be a random variable such that
\[
B_t := \begin{cases}
+1/K & \text{ with probability $2p_{i,t}(1-p_{i,t})$},\\
0 & \text{ with the remaining probability}.
\end{cases}
\]
Hence, in the cases $D_t=-1$ and $D_t=0$ we get $\Delta_t=B_t$. We call such
a step a \emph{biased step (b-step)} since $\expect{\Delta_t \mid p_{i, t}} = \expect{B_t \mid p_{i, t}} = 2p_{i,t}(1-p_{i,t})/K >0$ here.
\end{itemize}

Whether a step is an rw-step or b-step for bit~$i$ depends only on circumstances
being external to the bit (and independent of it). Let $R_t$ be the event
that $D_t=1$ or $\card{D_t\ge 2}$. We get the equality
\begin{equation}
\label{eq:superposition}
\Delta_t = F_t \cdot \prob{R_t} + B_t \cdot (1-\prob{R_t}),
\end{equation}
which we denote as \emph{superposition.} Informally, the change of $p_{i,t}$-value
is a superposition of a fair (unbiased) random walk and biased steps.
The fair random walk reflects the \emph{genetic drift} underlying the process, i.\,e.\ the variance in the process may lead the algorithm to move in a random direction. In contrast, the biased steps reflect steps where the algorithm \emph{learns} about which bit value leads to a better fitness at the considered bit position. We remark that the superposition of two different behaviors as formulated here is related to the approach taken in
\cite{ChenEtAlCEC2009}, where an EDA called UMDA was decomposed into a derandomized, deterministic EDA and
a stochastic component modeling genetic drift.

For \twoant, structurally this kind of superposition holds as well, however, the underlying
random variables look somewhat different. We have:

\begin{itemize}
\item If $\card{D_t}\ge 2$ or $D_t=1$, then the considered bit does not affect the choice whether to update
with respect to~$x$ or~$y$.
%; update
%will always be with respect to~$x_{t+1}$ if $D_t\ge 2$ and
%with respect to~$y_{t+1}$ if $D_t\le -2$.
 Hence, the marginal probability of the
considered bit
increases with probability~$p_{i,t}$ and decreases with probability~$1-p_{i,t}$.

We get $\Delta_t=p_{i,t+1}-p_{i,t}=F_t$ in this case, where
$F_t$ is a random variable
such that
\[
F_t := \begin{cases}
\rho\cdot (1-p_{i,t}) & \text{ with probability $p_{i,t}$,}\\
-\rho\cdot p_{i,t} & \text{ with probability $1-p_{i,t}$.}
\end{cases}
\]
 We call such a step an rw-step in analogy to \cga as here $\expect{\Delta_t \mid p_{i, t}} = \expect{F_t \mid p_{i, t}}=0$.

\item
If $D_t=0$ or $D_t=-1$ then the marginal probability  can only decrease if both offspring sample a~$0$ at bit~$i$;
otherwise it will increase. The difference $\Delta_t$ is a random variable
\[
B_t := \begin{cases}
\rho\cdot(1-p_{i,t}) & \text{with probability $1-(1-p_{i,t})^2$,}\\
-\rho\cdot p_{i,t} & \text{with probability $(1-p_{i,t})^2$.}
\end{cases}
\]
The step is called a biased step (b-step) as $\expect{\Delta_t \mid p_{i, t}} = \expect{B_t \mid p_{i, t}} = \rho p_{i, t}(1-p_{i, t})>0$.
\end{itemize}
Altogether, the superposition for \twoant is also given by \eqref{eq:superposition}, with the
modified meaning of $B_t$ and $F_t$.

The strength of the update plays a key role here: if the update is too strong, large steps are made during updates, and genetic drift through rw-steps may overwhelm the probabilistic model, leading to ``wrong'' decisions being made in individual bits. On the other hand, small updates imply that rw-steps have a bounded impact, and the algorithm receives more time to learn optimal bit values in b-steps.
We will formalize these insights in the following sections en route to proving rigorous upper and lower runtime bounds. Informally,
one main challenge is to understand the stochastic process induced by the mixture of b- and rw-steps.

\section{Small Update Strengths are Efficient}
\label{sec:upper}

We first show that small update strengths are efficient. This has been shown for \twoant in~\cite{Neumann2010a}.
\begin{theorem}[\cite{Neumann2010a}]
\label{the:upper-bound-twoant}
If $\rho \le  1/(cn^{1/2}\log n))$ for a sufficiently large constant~$c > 0$ and $\rho \ge 1/\poly(n)$
then \twoant optimizes \ONEMAX in expected time~$O(\sqrt{n}/\rho)$.

For $\rho = 1/(cn^{1/2}\log n)$ the runtime bound is~$O(n \log n)$.
\end{theorem}

Here we exploit the similarities between both algorithms to prove an analogous result for \cga.
\begin{theorem}
\label{the:cga-upper}
The expected optimization time of \cga on \OneMax with $K\ge c\sqrt{n}\log n$ for a sufficiently
large $c>0$ and $K = \poly(n)$ is $O(\sqrt{n}K)$. This is $O(n\log n)$ for $K = c\sqrt{n}\log n$.
\end{theorem}

The analysis follows the approach for \twoant in~\cite{Neumann2010a}, adapted to
the different update rule, and using modern tools like \emph{variable drift analysis}~\cite{Johannsen2010}.
The main idea is that marginal probabilities are likely to increase from their initial values of~$1/2$. If the update strength is chosen small enough, the effect of genetic drift (as present in rw-steps) is bounded such that with high probability all bits never reach marginal probabilities below $1/3$. Under this condition, we show that the marginal probabilities have a tendency (stochastic drift) to move to their upper borders, such that then the optimum is found with good probability.

The following lemma uses considerations and notation from Section~\ref{sec:superposition} to establish a \emph{stochastic drift}, i.\,e.\ a positive trend towards optimal bit values, for \cga.
We use the same notation as in Section~\ref{sec:superposition}.
\begin{lemma}
\label{lem:drift-for-bit-in-cga}
If $1/n + 1/K \le p_{i, t} \le 1 - 1/n - 1/K$ then
\[
\expect{\Delta_t \mid p_{i,t}} \ge \frac{2}{11} \frac{p_{i,t}(1-p_{i,t})}{K} \left(\sum_{j\neq i} p_{j,t}(1-p_{j,t})\right)^{-1/2}.
\]
\end{lemma}
%We do not give a full proof due to lack of space, but only give the main ideas.
\begin{proof}
The assumptions on $p_{i, t}$ assure that $p_{i, t+1}$ is not affected by the borders $1/n$ and ${1-1/n}$.
Then the expected change is given by the expectation of the superposition~\eqref{eq:superposition}:
\[
\expect{\Delta_t \mid p_{i,t}} = \expect{F_t \mid p_{i,t}} \cdot \prob{R_t} + \expect{B_t \mid p_{i,t}} \cdot (1-\prob{R_t}).
\]
%The expected change of rw-steps is $\expect{F_t \mid p_{i,t}} \cdot \prob{R_t} = 0$ since $\expect{F_t \mid p_{i,t}} = 0$.
%The expected change due to b-steps is bounded using
%%is $\expect{B_t \mid p_{i,t}} \cdot (1-\prob{R_t})$. By definition of $B_t$, %
%$\expect{B_t \mid p_{i,t}} = 2p_{i,t}(1-p_{i,t})/K$ and
From Section~\ref{sec:superposition} we know $\expect{F_t \mid p_{i,t}} = 0$ and
$\expect{B_t \mid p_{i,t}} = 2p_{i,t}(1-p_{i,t})/K$. Further,
\[
1-\prob{R_t} \ge \prob{D_t = 0} \ge \frac{1}{11} \left(\sum_{j\neq i} p_{j,t}(1-p_{j,t})\right)^{-1/2},
\]
where the last inequality was shown in~\cite[proof of Lemma~1]{Neumann2010a}. Here we exploit that \cga and \twoant use the same construction procedure. Together this proves the claim.
\end{proof}
Note that the term $\left(\sum_{j\neq i} p_{j,t}(1-p_{j,t})\right)^{1/2}$ reflects the standard deviation of the sampling distribution on all bits $j \neq i$.

Lemma~\ref{lem:drift-for-bit-in-cga} indicates that the drift increases with the update strength~$1/K$.
However, a too large value for $1/K$ also increases genetic drift.
The following lemma shows that, if $1/K$ is not too large, this positive drift implies that the marginal probabilities will generally move to higher values and are unlikely to decrease by a large distance.
\begin{lemma}
\label{lem:drift-theorem-for-probabilities}
Let $0 < \alpha < \beta < 1$ be two constants.
For each constant $\gamma > 0$ there exists a constant $c_\gamma > 0$ (possibly depending on $\alpha, \beta$, and $\gamma$) such that for a specific bit the following holds.
If the bit has marginal probability at least~$\beta$ and
$K \ge c_\gamma \sqrt{n} \log n$ then
the probability that during the following
$n^\gamma$~steps the marginal probability decreases below~$\alpha$ is at most $n^{-\gamma}$.
\end{lemma}
\begin{proof}
The proof is essentially the same as the proof of Lemma~3 in~\cite{Neumann2010a}, using $1/K$ instead of $\rho$ and drift bounds from Lemma~\ref{lem:drift-for-bit-in-cga}.
\end{proof}

With these lemmas, we now prove the main statement of this section.

\begin{proofof}{Theorem~\ref{the:cga-upper}}
We assume in the following that $1/K$ is a multiple of $1/2-1/n$, implying that marginal probabilities are restricted to $\{1/n, 1/n + 1/K, \dots, 1/2, \dots, 1-1/n-1/K, 1-1/n\}$.

Following~\cite[Theorem~3]{Neumann2010a} we show that, starting with a setting where all probabilities are at least~$1/2$ simultaneously,
with probability~$\Omega(1)$ after $O(\sqrt{n}K)$ iterations either
the global optimum has been found or at least one probability has dropped below~$1/3$.
In the first case we speak of a success and in the latter case  of a failure.
The expected time until either a success or a failure
happens is then $O(\sqrt{n}K)$.

Now choose a constant~$\gamma > 0$ such that $n^\gamma \ge K n^3$.
According to Lemma~\ref{lem:drift-theorem-for-probabilities} applied with $\alpha := 1/3$ and $\beta := 1/2$, the probability of a failure
in $n^{\gamma}$ iterations is at most $n^{-\gamma}$, provided the constant~$c$ in the condition $K \ge c\sqrt{n}\log n$ is large enough.
In case of a failure we wait until the probabilities
simultaneously reach values at least~$1/2$ again and then we repeat the arguments from the preceding paragraph.
It is easy to show (cf.\ Lemma~2 in~\cite{Neumann2010a}) that the expected time for one probability to reach the upper border is always bounded by $O(n^{3/2}K)$, regardless of the initial probabilities. By standard arguments on independent phases, the expected time until \emph{all} probabilities have reached their upper border at least once is $O(n^{3/2}K \log n)$.
Once a bit reaches the upper border, we apply Lemma~\ref{lem:drift-theorem-for-probabilities} again with $\alpha := 1/2$ and $\beta := 2/3$ to show that the probability of a marginal probability decreasing below $1/2$ in time $n^{\gamma}$ is at most $n^{-\gamma}$ (again, for large enough~$c$). The probability that there is a bit for which this happens is at most $n^{-\gamma + 1}$ by the union bound. If this does not happen, all bits attain value at least $1/2$ simultaneously, and we apply our above arguments again.

As the probability of a failure is at most $n^{-\gamma+1}$, the expected number of restarts is $O(n^{-\gamma+1})$
and considering the expected time until all bits recover to values at least $1/2$ only leads to
an additional term of $n^{-\gamma+1} \cdot O((n^{3/2} \log n)K) \le  o(1)$ (as $n^{-\gamma} \le n^{-3}/K$) in the expectation.

We only need to show that after $O(\sqrt{n}K)$ iterations without failure
the probability of having found the global optimum is~$\Omega(1)$.
To this end, we consider a simple potential function that takes into account marginal probabilities for all bits.
An important property of the potential is that once the potential has decreased to some constant value, the probability of generating the global optimum is constant.

Let $p_1, \dots, p_n$ be the current marginal probabilities and $q_i := 1-1/n-p_i$ for all~$i$.
Define the potential function $\pot := \sum_{i=1}^n q_i$, which measures the distance to an ideal setting where all probabilities attain their maximum~$1-1/n$.
Let $q_i'$ be the $q_i$-value in the next iteration and $p_i' = 1-q_i'$.
We estimate the expectation of~$\pot' := \sum_{i=1}^n q_i'$ and distinguish between two cases.
If $p_i \le 1-1/n-1/K$,
by Lemma~\ref{lem:drift-for-bit-in-cga}
\begin{align*}
\E{q_i' \mid q_i} \;&\le\; q_i - \frac{p_i(1-p_i)}{K} \cdot \frac{2}{11} \cdot \left(\sum_{j \neq i} p_j(1-p_j)\right)^{-1/2}.
\end{align*}
We bound $p_i(1-p_i)$ from below using $p_{i} \ge 1/3$ and $1-p_i \le 1-1/n-p_i = q_i$ and the sum from above using
\[
\sum_{j \neq i} p_j(1-p_j) \le \sum_{j=1}^n (1-p_j) = \sum_{j=1}^n (q_j+1/n) = 1 + \pot.
\]
Then
\begin{align*}
\E{q_i' \mid q_i} &\le\; q_i - \frac{q_i}{K} \cdot \frac{2}{33} \cdot \left(\frac{1}{1 + \pot}\right)^{1/2}\\
&\le\; q_i \left(1 - \frac{2}{33K} \cdot \frac{1}{1 + \pot^{1/2}}\right).
\end{align*}

If $p_i > 1-1/n-1/K$, then $p_i = 1-1/n$ (as $1/K$ is a multiple of $1/2-1/n$) and $p_i$ can only decrease. A decrease by~$1/K$
happens with probability~$1/n$, thus
\[
\E{q_i' \mid q_i} \;\le\; q_i + \frac{1}{nK}.
\]
To ease the notation we assume \wlo\ that the bits are numbered according to decreasing probabilities, \ie, increasing $q$-values. Let $m \in \N_0$ be the largest index such that
$p_{m} = 1-1/n$. It follows
\[
\sum_{i=1}^{m} \E{q_i' \mid q_i} \;\le\; \sum_{i=1}^m q_i + \frac{m}{nK}
\le \sum_{i=1}^m q_i + \frac{1}{K}.
\]
Putting everything together and using $\sum_{i=1}^m q_i = \frac{m}{n} \le 1$,
\begin{align*}
\E{\pot' \mid \pot}% \;&=\; \sum_{i=1}^n \E{q_i' \mid q_i}\\
\;&=\; \sum_{i=1}^m \E{q_i' \mid q_i} + \sum_{i=m+1}^n \E{q_i' \mid q_i}\\
\;&\le\; \sum_{i=1}^m q_i + \frac{1}{K} + \sum_{i=m+1}^n q_i \left(1 - \frac{2}{33K} \cdot \frac{1}{1 + \pot^{1/2}}\right)\\
%\intertext{(using $\sum_{i=1}^m q_i \le \frac{1}{K}$)}
\;&\le\; 1 + \frac{1}{K} + \left(\pot - 1\right)\left(1 - \frac{2}{33K} \cdot \frac{1}{1 + \pot^{1/2}}\right)\\
\;&\le\; \pot \left(1 - \frac{2}{33K} \cdot \frac{1}{1 + \pot^{1/2}}\right) + \frac{3}{K}
\end{align*}
where in the last line we used $\frac{2}{33K} \cdot \frac{1}{1 + \pot^{1/2}} \le \frac{2}{33K} \le 2/K$.
For $\pot \ge 10000$ this can further be bounded using $1 + \pot^{1/2} \le \pot^{1/2}/100 + \pot^{1/2} = 101/100 \cdot \pot^{1/2}$,
\begin{align*}
\E{\pot' \mid \pot}
\le \pot - \pot^{1/2} \cdot \frac{101}{100} \cdot \frac{2}{33K} + \frac{3}{K}
\le \pot - \pot^{1/2} \cdot \frac{101}{3300K}
\end{align*}
where in the last step we used $\pot^{1/2} \cdot \frac{101}{100} \cdot \frac{1}{33K} \ge \frac{101}{33K} \ge \frac{3}{K}$, \ie, half of the negative term subsumes the $\mathrel{+}3/K$ term.

Now a straightforward generalization of variable drift theorem (given by Theorem~\ref{drift:johannsen-general} in the appendix),
applied with a drift function of $h(\pot) := \pot^{1/2} \cdot \frac{101}{3300K}$, states that the expected time for $\pot$ to decrease from any initial value $\pot \le n$ to a value $\pot \le 10000$ is at most
\begin{align*}
& \frac{10000}{h(10000)} + \int_{10000}^{n} \frac{1}{h(\pot)} \;\mathrm{d}\pot\\
=\;& O(K) + O(K) \cdot \int_{10000}^{n} \pot^{-1/2} \;\mathrm{d}\pot
 = O(\sqrt{n}K).
\end{align*}
Consider an iteration where~$\pot \le 10000$.
The probability of creating ones on all bits simultaneously, given that all marginal probabilities
are at least~$1/3$,
is minimal in the extreme setting where a maximal number of bits has marginal probabilities
at~$1/3$ and all other bits, except at most one, have marginal probabilities at
their upper border. Then the probability of creating the optimum in one step is
at least
$
\left(1-\frac{1}{n}\right)^{n-1} \cdot 3^{-\lceil \pot \cdot 3/2 \rceil}
= \Omega(1).
$
Hence a successful phase finds the optimum with probability $\Omega(1)$.%, completing the proof.
%If this does not happen, we wait until we reach a potential of $\pot \le 10000$ again; this time is again bounded by $O(\sqrt{n}K)$.
%The expected time for finding the optimum is hence bounded by $O(\sqrt{n}K)$.
\end{proofof}

\section{Large Update Strengths Lead to Genetic Drift}
\label{sec:lower}

The bound $O(\sqrt{n}K)$ from  Theorem~\ref{the:cga-upper} shows that
larger update strengths (\ie, smaller $K$) result in smaller bounds on the runtime.
However, the theorem requires that $K\ge c\sqrt{n}\log n$ so that the best possible
choice results in $O(n\log n)$ runtime. An obvious question to ask is whether
this is only a weakness of the analysis or whether there is an intrinsic limit that
prevents smaller choices of $K$ from being efficient.

In this section, we will show that smaller choices of $K$ (\ie, larger update strengths) cannot give runtimes
of lower orders than $n\log n$. In a nutshell, even though larger update strengths support faster exploitation
of correct decisions at single bits
by quickly reinforcing promising bit values
they also increase the risk of genetic drift reinforcing incorrectly made decisions at single bits too quickly.
Then it typically happens that several marginal probabilities reach their lower border~$1/n$, from which it (due to
 so-called
coupon collector effects) takes
$\Omega(n\log n)$ steps to ``unlearn'' the wrong settings. The very same effect happens with \twoant if
its update strength $\rho$ is chosen too large.

We now state the lower bounds we obtain for the two algorithms, see Theorems~\ref{theo:lower-cga} and~\ref{theo:lower-mmas} below.
Note that the statements
are identical
if we identify the update strength $1/K$ of~\cga with the update strength $\rho$ of~\twoant. Also the proofs
of these two theorems
will largely follow the same steps. Therefore, we describe the proof approach in detail with respect to~\cga
in Section~\ref{sec:proof-lower-cga}. In Section~\ref{sec:proof-lower-mmas}, we describe the few places
where slightly different arguments are needed to obtain the result for \twoant.

\begin{theorem}
\label{theo:lower-cga}
The optimization time of \cga with $K \le \poly(n)$ is $\Omega(\sqrt{n}K + n \log n)$ with probability $1-\poly(n) \cdot 2^{-\Omega(\min\{K, n^{1/2-o(1)}\})}$ and in expectation.
\end{theorem}

\begin{theorem}
\label{theo:lower-mmas}
The optimization time of \twoant with $1/\rho \le \poly(n)$ is $\Omega(\sqrt{n}/\rho + n \log n)$ with probability $1-\poly(n) \cdot 2^{-\Omega(\min\{1/\rho, n^{1/2-o(1)}\})}$ and in expectation.
\end{theorem}

\subsection{Proof of Lower Bound for \cga}
\label{sec:proof-lower-cga}

We first describe at an intuitive level why large update strengths in \cga can be risky.
In the upper bound from Theorem~\ref{the:cga-upper}, we have shown that for sufficiently small update strengths, the positive stochastic drift by b-steps is strong enough such that even in the presence of rw-steps \emph{all} bits never reach marginal probabilities below $1/3$, with high probability. Then no ``incorrect'' decision is made.

%Then the drift due to b-steps would enable the algorithm
%to learn the optimal settings in $O(\sqrt{n}K)$ time per bit.

To prove Theorem~\ref{theo:lower-cga}, we show that with larger update strengths than $1/(\sqrt{n}\log n)$ the effect of rw-steps is strong enough such that with high probability \emph{some} bits will make an incorrect decision and reach the lower borders of marginal probabilities.
We consider the hitting time for a marginal probability to reach the lower border $1/n$ and analyze the distribution of this hitting time more closely.

To illustrate this setting, fix one bit and imagine that all steps were rw-steps (we will explain later how to handle b-steps), and that all rw-steps change the current value of the bit's marginal probability (\ie, there are no self-loops).
Then the process would be a fair
random walk on $\{0,1/K,2/K,\dots,(K-1)/K,1\}$,  started at $1/2$.
This fair random walk is well understood and it is well known that the hitting time
is not sharply concentrated around the expectation. More precisely,
there is still a polynomially in $K$ small probability of hitting a border within at most $O(K^2/\!\log K)$ steps
and also of needing at least $\Omega(K^2\log K)$ steps. The underlying idea
is that the Central
Limit Theorem (CLT) approximates the progress within a given number of steps.

%This process is well-understood: the expected time for hitting a border is $K^2/4$ steps, and the probability of finishing early, that is, hitting a border in time $\alpha K^2/4$, for $0 < \alpha < 1$, is exponential in
%$-\Omega(1/\alpha)$. This statement would be sufficient to show that very likely some bits may reach the lower border in time $O(K^2/\log n)$.
The real process is more complicated because of self-loops. Recall from the definition of $F_t$ that the process only changes its current state by $\pm 1/K$
with probability $2p_{i, t}(1-p_{i, t})$, hence with probability
$1-2p_{i,t}(1-p_{i, t})$ a self-loop occurs on this bit.
The closer the process is to one of its borders $\{1/n,1-1/n\}$, the  larger the self-loop probability becomes and the more the random walk slows down. Hence the actual process is clearly
slower in reaching a border since every looping step is just wasted. %(other bits may continue to evolve towards optimal values).
One might conjecture that
the self-loops will asymptotically increase the expected hitting time. But interestingly, as we will show, the expected hitting time in the presence of self-loops is still of order $\Theta(K^2)$.
Also the CLT (in a generalized form) is still applicable despite the self-loops, leading to a similar distribution as above.

%Intuitively, in the end
%this lack of concentration will result in a number of ``unusually fast'' marginal probabilities,
%which reach their lower border (\ie, a wrong setting)
%in less than the expected number of $\Theta(K^2)$ steps. Basically,
%these unusually fast random walks prevent us from setting $K=\Theta(\sqrt{n})$ and lead to the optimal
%choice $K=\Theta(\sqrt{n}\log n)$, resulting in a $\log n$-factor in the expected
%optimization time.

The distribution of the hitting time of the random walk with self-loops will be
analyzed in Lemma~\ref{lem:distribution-cga-self-loop} below. In order to deal with self-loops, in its proof,
 we use a potential function mapping the actual process to a process on a scaled
state space
with nearly
position-independent variance.
 Unlike the typical applications of potential functions in drift analysis, the purpose of the
potential function is
not to establish a position-independent first-moment stochastic drift  but a (nearly)  position-independent
variance, \ie, the potential function is designed to analyze a second moment. This argument
seems to be new in the theory of drift analysis and may be of independent interest.
The lemma also takes
into account the b-steps in between rw-steps and shows how the rw-steps can still overwhelm the accumulated
effect of b-steps if the latter are not too frequent.
\begin{lemma}
\label{lem:distribution-cga-self-loop}
Consider a bit of \cga on \OneMax and let $p_t$ be its marginal probability at time~$t$. Let
$t_1, t_2, \dots$ be the times where \cga performs an rw-step
(before hitting one of the borders $1/n$ or $1-1/n$) and let $\Delta_i:=p_{t_i+1}-p_{t_i}$.
For $s\in \R$,
let $T_s$ be the smallest $t$ such that
$\sgn(s)\left(\sum_{i=0}^{t} \Delta_{i}\right) \ge \card{s}$ holds or
a border has been reached.

Choosing  $0<\alpha<1$, where $1/\alpha=o(K)$, and
$-1<s<0$ constant, and assuming that at most
 $\card{s}K/4$ of the steps until
time $t_{\alpha (sK)^2}$ are b-steps, we have
\begin{align*}
& \prob{T_s\le \alpha (sK)^2 \text{ or $p_t$ exceeds $5/6$ before~$T_s$}} \\
& \qquad\qquad\qquad \ge (1/2-o(1)) \cdot
 \Bigl(\frac{1}{13\sqrt{1/(\card{s} \alpha)}}-\frac{1}{(13\sqrt{1 /(\card{s}\alpha)})^{3}}\Bigr)\frac{1}{\sqrt{2\pi}}e^{-\frac{169}{2\card{s}\alpha}}.
\end{align*}

Moreover, for any $\alpha>0$ and $s\in\R$,
\[\prob{T_s\ge \alpha (sK)^2 \text{ or a border is reached until time $\alpha (sK)^2$}}
\ge 1-e^{-1/(4\alpha)}.\]
\end{lemma}

Informally, the  lemma means that every deviation of the
hitting time $T_s$ by a constant factor from
its expected value (which turns out as $\Theta(s^2K^2)$)
still has constant probability, and even deviations by logarithmic factors
have a polynomially small probability. We will mostly apply the lemma
for $\alpha<1$, especially $\alpha \approx 1/\!\log n$, to show that there
are marginal probabilities that quickly approach the lower border; in fact,
this effect implies the
 $\log n$ term in the optimal update strength.
Note that the second statement
of the lemma also holds  for $\alpha\ge 1$; however, in this realm
also Markov's inequality works.
 Then, by
the inequality $e^{-x}\le 1-x/2$ for $x\le 1$, we get $\prob{T_s\ge \alpha s^2K^2}\ge 1/(4\alpha)$, which
means that Markov's inequality for deviations above the expected value is asymptotically tight in this case.

To illustrate the main idea for the proof of Lemma~\ref{lem:distribution-cga-self-loop},
we ignore b-steps for a while and
note that we are confronted with a fair random walk. However,
the random walk is not longer homogeneous with respect to place as the self-loops slow the process down
in the vicinity of a border. The random variables describing
the change of position from time~$t$ to time~$t+1$ (formally, $\Delta_t:=p_{t+1}-p_{t}$)
that are not identically distributed, other than in the classical fair random walk.
In fact, the variance of $\Delta_t$ becomes
smaller the closer $p_t$ is to one of the borders.

In more detail, the potential function used in  Lemma~\ref{lem:distribution-cga-self-loop}
essentially uses the self-loop probabilities to construct extra distances to bridge. For instance,
states with low self-loop probability (\eg, $1/2$), will have a potential
that is only by $\Theta(1)$ larger or smaller than the potential of its neighbors. On the other hand,
states with a large self-loop probability, say $1/K$, will have a potential that can differ
by as much as $2\sqrt{K}$ from the potential of its neighbors. Interestingly, this
choice leads to variances of the one-step changes that are basically the same on the
whole state space (very roughly, this is true since the squared change $(2\sqrt{K})^2=\Theta(K)$ is
observed with probability $\Theta(1/K)$).  However, using the potential
for this trick
 is at the expense of changing the support of the
underlying random variables, which then will depend on the state. Nevertheless, as the support is not changed
too much, the Central Limit Theorem (CLT) still applies and we can approximate the progress made
within $T$ steps by a normally distributed random variable. This approximation is
made precise in the following lemma. See, Eq.~(27.16) in \cite{Billingsley1995}.

\begin{lemma}[Weak CLT with Lyapunov condition]
\label{lem:weak-clt}
Let $X_1,\dots,X_n$ be  a sequence of independent random variables, each with finite expected value $\mu_i$ and variance $\sigma_i^2$. Define
\[
s_n^2 := \sum_{i=1}^n \sigma_i^2
\text{\quad and\quad}
    C_n := \frac{1}{s_n^2}\sum_{i=1}^{n} (X_i-\mu_i).
		\]
If there exists  $\delta>0$ such that
\[
\lim_{n\to \infty} \frac{1}{s_n^{2+\delta}}\sum_{i=1}^{n} \expect{\lvert X_i - \mu_i\rvert^{2+\delta}} = 0
\]
(assuming all the moments of order $2+\delta$ to be defined),
then
$C_n$ converges in distribution to a standard normally distributed random variable,
more precisely for all $x\in \R$
\[
\lim_{n\to \infty} \prob{ C_n \le x} = \Phi(x),\]
 or, equivalently,
\[
\left\lvert\prob{ C_n \le x} - \Phi(x)\right\rvert=o(1),
\]
where $\Phi(x)$ is the cumulative distribution function of the standard normal distribution.
\end{lemma}

We now turn to the formal proof.

\begin{proofof}{Lemma~\ref{lem:distribution-cga-self-loop}}
Throughout this proof, to ease notation we consider the scaled process on the state space $S:=\{0,1,\dots,K\}$
obtained by multiplying all marginal probabilities by~$K$; the random variables
$X_t=K p_{t}$ will live on this scaled space. Note that we also remove the borders ($K/n$ and $K-K/n$),
which is possible as all considerations are stopped when such a border is reached. For the same
reason, we only consider current states from $\{1,\dots,K-1\}$ in the remainder of this proof.

The first hitting time $T_s$
becomes only stochastically larger if we ignore all self-loops. Formally, recalling the trivial scaling
of the state space, we
consider the  fair random walk where $\prob{X_{t_i+1}=j-1}=\prob{X_{t_i+1}=j+1}=1/2$
if $X_{t_i}=j\in\{1,\dots,K-1\}$.
We write $Y_t=\sum_{i=0}^{t-1} \Delta_{t_i}$.
Clearly, $\Delta_i$ is uniform on $\{-1,1\}$, $\expect{\Delta_i\mid 0<X_{t_i}<K}=0$, $\Var(\Delta_i\mid 0<X_{t_i}<K)=1$
 and $Y_t$ is a sum of independent, identically distributed random variables. It is well
known that $(Y_t-\expect{Y_t})/\sqrt{\Var(Y_t)}$ converges in distribution to a standard normally
distributed random variable. However, we do not use this fact directly here.
Instead, to bound the deviation from the expectation, we use a classical Hoeffding bound. We assume
$s\ge 0$ now and will see that the case $s<0$ can be handled symmetrically.

Theorem~1.11
in~\cite{DoerrTools} yields, with $c_i=2$
as the size of the support of $\Delta_i$, that
\[
\prob{Y_{\alpha s^2K^2}\ge sK} \le e^{-(sK)^2 / (4\alpha s^2K^2)} = e^{-1/(4\alpha)}. \]
Moreover, according to Theorem~1.13 in \cite{DoerrTools}, the bound also holds
for all $k\le \alpha s^2 K^2$ together, more precisely,
\[
\prob{\exists k\le \alpha s^2 K^2\colon Y_k \ge sK} \le e^{-1/(4\alpha)}.
\]
Symmetrically, we obtain
\[
\prob{\exists k\le \alpha s^2 K^2\colon Y_k \le -sK} \le e^{-1/(4\alpha)}.
\]
Hence, distance that is
strictly smaller than $sK$ is bridged through $\alpha (sK)^2$ rw-steps
(or the process reaches a border before)
with probability at least~$1-e^{-1/(4\alpha)}$.

We are left with the first statement, where the stronger condition
$-1<s<0$ and $\card{s}=\Omega(1)$ is made. Here we will essentially use an approximation
of the accumulated state within $\alpha s^2K^2$ steps
by the normal distribution, but have to be careful to take into account
steps describing self-loops.
To analyze the hitting time $T_s$ for the $X_{t_i}$-process, we now define a potential
function $g\colon S\to \R$.
Unlike the typical applications of potential functions, the purpose of~$g$ is
not to establish a position-independent first-moment drift (in fact, there is no drift
within $S$ since the original process is a martingale)  but a (nearly)  position-independent
variance, \ie, the potential function is designed to analyze a second moment.

%% This paragraph was repeated from a discussion given before the proof.
%%
%In a nutshell,
%the potential function essentially uses the self-loop probabilities to construct extra distances
%to bridge.
%For instance,
%states with low self-loop probability (\eg, $K$), will have a potential
%that is only by $\Theta(1)$ larger or smaller than the potential of its neighbors. On the other hand,
%states with a large self-loop probability, say $2/K$, will have a potential that can differ
%by as much as $2\sqrt{K}$ from the potential of its neighbors. Interestingly, this
%choice leads to variances of the one-step changes that are basically the same on on the
%whole state space (very roughly, this is true since the squared change $(2\sqrt{K})^2=\Theta(K)$ is
%observed with probability $\Theta(1/K)$).
%However, using the potential
%in this way
 %is at the expense of changing the support of the
%underlying random variables ($X_t-X_{t+1}$ is mapped to $g(X_t)-g(X_{t+1})$), which then will depend on the state. Nevertheless, as the support is not changed
%too much, the Central Limit Theorem still  (CLT) applies and we can approximate the progress made
%within $T$ steps by a Normally distributed random variable.

We proceed with the formal definition of the potential function, the analysis of its expected
first-moment change and the corresponding variance, and a proof that the Lyapunov condition
holds for the accumulated change within $\alpha s^2K^2$ steps.
 The potential function~$g$ is monotonically decreasing on $\{1,\dots,K/2\}$ and centrally symmetric around $K/2$.
We define it as follows:
let $g(K/2)=0$ and for $1\le i\le K/2-1$, let $g(i)-g(i+1) = \sqrt{2K/(i+1)}$; finally, let
$g(K-i)=-g(i)$.
Inductively, we have
\[
g(i) = -g(K-i) =
\sum_{j=i}^{K/2-1} \sqrt{2K/(j+1)}
\]
for $1\le i\le K/2$.
We note that $g(0)=O(K)$, more precisely it holds
\[g(0) = \sqrt{2K}(\sum_{j=1}^{K/2-1} \sqrt{1/(j+1)}) \le \sqrt{2K}(2\sqrt{K/2})=2K.\]
 More generally,
for $i<j\le K/2$, we get by the monotonicity of~$g$ that
\begin{equation}
g(i)-g(j) \le g(0)-g(j-i) = \sqrt{2K}\sum_{k=1}^{j-i}
\sqrt{1/k} \le 2\sqrt{2K}(\sqrt{j-i})
\label{eq:pot-stretch}
\end{equation}
 Informally,
the potential function stretches the whole state space by a factor of at most~$2$ but adjacent
states in the vicinity of borders can be by $2\sqrt{K}$ apart in potential.

Let $Y_t:=g(X_t)$. We consider the one-step differences
$\Psi_i:=Y_{t_i+1}-Y_{t_i}$  at the times
$i$ where rw-steps occur,   and we will show via the representation $Y_{t_i}:=\sum_{j=0}^{i-1} \Psi_j$ that $Y_{t_i}$ approaches
a normally distributed variable. Note that $Y_{t_i}$ is not necessarily
the same as $g(X_{t_i})-g(X_{t_0})$ since only the effect of rw-steps is covered by
$Y_{t_i}$.

In the following, we assume $1\le X_{t_i}\le K/2$ and note that the case $X_{t_i}>K/2$ can be handled
symmetrically with respect to $-\Psi_i$.
We claim that for all $i\ge 0$
\begin{align}
0 &\le \expect{\Psi_i\mid X_{t_i}} \le \sqrt{2/(X_{t_i}K)} \le o(1) \label{eq:expectpsiti},\\
1/4 & \le \Var(\Psi_i \mid X_{t_i}),
\end{align}
where all $O$-notation is with respect to~$K$.

The lower bound $\expect{\Psi_i\mid X_{t_i}}\ge 0$ is easy to see
since $X_{t_i}$ is a fair random walk and $g(j-1)-g(j) \ge
g(j)-g(j+1)$ holds for all $j\le K/2$. To prove the upper bound,
we note that $X_{t_i+1}\in\{X_{t_i}-1,X_{t_i},X_{t_i}+1\}$
so that
\[\expect{\Psi_i\mid X_{t_i}} =
\prob{X_{t_i+1}< X_{t_i}} (g(X_{t_{i}}-1)-g(X_{t_i}))
+ \prob{X_{t_i+1}> X_{t_i}} (g(X_{t_{i}}+1)-g(X_{t_i}) )
\]
% Let $X_{t_i}=\max\{1,X_{t_i}\}$,
%$X_{t_i}minus=\max\{1,X_{t_i}-1\}$ and $X_{t_i}+1=\max\{1,X_{t_i}+1\}$.
Using the properties of rw-steps, we have that
$\prob{Y_{t_i+1}\neq Y_{t_i}} = 2\frac{(K-X_{t_i}) X_{t_i}}{K^2}$. Moreover,
on $Y_{t_i+1}\neq Y_{t_i}$, $Y_{t_i+1}$ takes each of the two values $g(X_{t_i}-1)$ and $g(X_{t_i}+1)$
with the same probability.
Hence
\begin{align*}
\expect{\Psi_i\mid X_{t_i}} =\;&
\frac{(K-X_{t_i}) X_{t_i}}{K^2} \left((g(X_{t_{i}}-1)-g(X_{t_i}))
+ (g(X_{t_{i}}+1)-g(X_{t_i}) )\right)\\
=\;&
\frac{(K-X_{t_i}) X_{t_i}}{K^2} \left((g(X_{t_{i}}-1)-g(X_{t_i}))
- (g(X_{t_i})- g(X_{t_{i}}+1))\right)\\
=\;&
\frac{(K-X_{t_i}) X_{t_i}}{K^2} \cdot \sqrt{2K}\left(\frac{1}{\sqrt{X_{t_i}}}-\frac{1}{\sqrt{X_{t_i}+1}}\right)\\
\le\;&
\frac{X_{t_i}}{K} \cdot \sqrt{2K}\left(\frac{1}{\sqrt{X_{t_i}}}-\frac{1}{\sqrt{X_{t_i}+1}}\right).
\end{align*}
We estimate the bracketed terms using
\begin{align*}
& \frac{1}{\sqrt{X_{t_i}}}-\frac{1}{\sqrt{X_{t_i}+1}}
 = \frac{\sqrt{X_{t_i}+1}-\sqrt{X_{t_i}}}{\sqrt{X_{t_i}}\sqrt{X_{t_i}+1}}
 \le \frac{1/(2\sqrt{X_{t_i}})}{X_{t_i}} \le \frac{1}{\left(X_{t_i}\right)^{3/2}},
\end{align*}
where the last inequality exploited that $f(x+h)-f(x)\le h f'(x)$ for any concave, differentiable
function $f$ and $h\ge 0$; here using $f(x)=\sqrt{x}$ and $h=1$.
Altogether, \[
\expect{\Psi_i\mid X_{t_i}} \le \frac{X_{t_i}}{K} \cdot \frac{\sqrt{2K}}{\left(X_{t_i}\right)^{3/2}} =
\frac{\sqrt{2} X_{t_i}}{\sqrt{K}\left(X_{t_i}\right)^{3/2}} \le \sqrt{\frac{2}{X_{t_i}K}},
\]
which proves \eqref{eq:expectpsiti} since $X_{t_i}\ge 1$ and $K=\omega(1)$.

To verify the bound on the variance, note that
\begin{align*} \Var(\Psi_{i} \mid X_{t_i}) &
\ge \expect{(\Psi_i - \expect{\Psi_{i}\mid X_{t_i}})^2 \cdot
\indic{\Psi_{i}\le 0}\mid X_{t_i} }\\
& \ge \expect{(\Psi_i )^2 \cdot
\indic{\Psi_{i}\le 0}\mid X_{t_i} }
\end{align*}
since $\expect{\Psi_i\mid X_{t_i}} \ge 0$. Now, as $0<X_{t_i}\le K/2$, we have
$\prob{Y_{t_i+1}<Y_{t_i}} = \frac{(K-X_{t_i}) X_{t_i}}{K^2} \ge \frac{X_{t_i}}{2K}$.
Moreover,
$Y_{t_i+1}<Y_{t_i}$ implies that $X_{t_i+1}=X_{t_i}+1$ since $g$ is monotone decreasing
on $\{1,\dots,K/2\}$
and the $X_{t_i}$-value can change by either $-1$, $0$, or $1$. Hence,
if $Y_{t_i+1}<Y_{t_i}$ then $Y_{t_i+1}-Y_{t_i} = g(X_{t_i}+1) - g(X_{t_i}) = - \sqrt{2K/(X_{t_i}+1)}$.
Altogether,
\[
\Var(\Psi_i \mid X_{t_i}) \ge \frac{X_{t_i}}{2K} \cdot \left(- \sqrt{2K/(X_{t_i}+1)}\right)^2 \ge 1/4,
\]
where we used $X_{t_i}/(X_{t_i}+1)\ge 1/2$.
This proves the lower bound on the variance.

We are almost ready to prove that $Y_{t_i}:=\sum_{j=0}^{i-1} \Psi_j$ can be approximated by a normally distributed
random variable for sufficiently large~$t$. We denote by $s_i^2 := \sum_{j=0}^{i-1} \Var(\Psi_j\mid X_{t_j})$ and note
that $s_i^2 \ge i/4$ by our analysis of variance  from above. The so-called Lyapunov condition, which is
sufficient for convergence
to the normal distribution (see Lemma~\ref{lem:weak-clt}),
requires the existence of some $\delta>0$ such that
\[
\lim_{i\to \infty} \frac{1}{s_i^{2+\delta}}\sum_{j=0}^{i-1} \expect{\lvert\Psi_j - \expect{\Psi_j\mid X_{t_j}}\rvert^{2+\delta}\mid X_{t_j}} = 0.
\]
We will show that the condition is satisfied for $\delta=1$ (smaller values could be used but
do not give any benefit) and $i=\omega(K)$ (which, as $i=\alpha s^2K^2$, holds due to our
assumptions $1/\alpha=o(K)$ and $\card{s}=\Omega(1)$).
We argue that
\[\card{\Psi_i-\expect{\Psi_i\mid X_{t_i}}} \le \card{\Psi_i}+\card{\expect{\Psi_i\mid X_{t_i}}}
 \le \card{\max\{k\mid \prob{\card{\Psi_i}\ge k\mid X_{t_i}}>0\}} + o(1),\] where
we have used the bound on $\card{\expect{\Psi_i\mid X_{t_i}}}$ from \eqref{eq:expectpsiti}. As the $X_{t_i}$-value can only
change by $\{-1,0,1\}$, we get, by summing up all possible changes of the $g$-value, that
\begin{align*}
\card{\Psi_i-\expect{\Psi_i\mid X_{t_i}}} \le\;& (g(X_{t_i}-1)-g(X_{t_i})) + (g(X_{t_i})-g(X_{t_i}+1)) + o(1)\\
\le\;&
g(X_{t_i}-1) - g(X_{t_i}+1) + o(1)\\
\le\;&
\left(2\cdot \sqrt{2K/(X_{t_i}-1)}\right) + o(1)%\le 3\sqrt{2K/\max\{1,X_{t_i}-1\}}
\end{align*}
for
$K$ large enough.

Hence, plugging this in the Lyapunov condition,
\[
\expect{\lvert\Psi_j - \expect{\Psi_j\mid X_{t_j}}\rvert^{3}\mid X_{t_j}} \le
 \frac{2X_{t_j}}{K}  \left(2\cdot \sqrt{2K/(X_{t_j}-1)}\right)^3 (1 + o(1))+o(1)= O(\sqrt{K}),
\]
implying that
\[
\frac{1}{s_i^{3}}\sum_{j=0}^{i-1} \expect{\lvert\Psi_j - \expect{\Psi_j}\rvert^{3}\mid X_{t_j}}
\le \frac{1}{(i/4)^{1.5}} O(i\sqrt{K}) = O(\sqrt{K/i}),
\]
which goes to $0$ as $i=\omega(\sqrt{K})$. Hence, for the value $i:=\alpha s^2 K^2$ considered in the lemma
we obtain that $\frac{Y_{t_i}-\expect{Y_{t_i}\mid X_0}}{s_i}$ converges in distribution
to  $N(0,1)$. Note that $s_i^2 \ge \alpha s^2 K^2/4$ by our analysis of variance and therefore
$s_i\ge \sqrt{\alpha}\card{s}K/2$. We have to be careful when computing $\expect{Y_{t_i}}$ since
$\expect{\Psi_i\mid X_{t_i}}$ is negative for $X_{t_i}>K/2$. Note, however, that
considerations are stopped when the marginal probability exceeds~$5/6$, \ie, when
$X_{t_i}>5K/6$. Using \eqref{eq:expectpsiti}, we hence have that $\expect{\Psi_i\mid X_{t_i}}\ge -\sqrt{2/(5K^2/6)}
\ge -1.55/K$. Therefore, $\expect{Y_{t_i}} \ge i\cdot (-1.55/K) = -1.55\alpha s^2 K$ and  $\expect{Y_{t_i}/s_i} \ge
-3.1\card{s}\sqrt{\alpha}$.

Hence, using the approximation by the normal distribution
and taking into account the scaling and shifting, we have that
\begin{align}
& \prob{Y_{t_i}\ge  rK} \notag\\
& \; \ge (1- o(1))(1-\Phi(rK/s_i-\expect{Y_{t_i}/s_i})) = (1- o(1))(1-\Phi(r/(\card{s}\sqrt{\alpha/4})+3.1\card{s}\sqrt{\alpha}))
\label{eq:prob-normalized-cga}
\end{align}
for any $r$ leading to a positive argument of~$\Phi$,
where $\Phi$ denotes the cumulative distribution function of the standard
normal distribution.

Recall that our aim is to bound $\sum_{j=0}^{i-1} \Delta_j = X_{t_i}-X_0$.  To this end,
we look into the event that $Y_{t_i}\ge 3\sqrt{\card{s}}K$ (noting that $s<0$) and
study $Y_{t_i}-g(X_{t_i})<0$, which reflects the accumulated effect of b-steps
on the potential function until time~$t_i$ (recall that a b-step increases the $X_t$-value and
decreases the $g(X_t)$-value. Given $X_t=x$ and assuming a b-step at time~$t$,
we have $X_{t+1}>x$ with probability at most~$x/K$. Hence, $g(X_{t+1})-g(x)
\ge -\frac{x}{K} \frac{2\sqrt{K}}{\sqrt{x}} \ge -2$ since $x\le K/2$. Let
$B:=\{j\le t_i\mid j\notin\{t_1,\dots,t_i\}\}$ be the indices of the b-steps until time~$t_i$.
Since by assumption only $\card{s}K/4$ b-steps occur until time~$t_i$, we get
$\sum_{j\in B}\expect{ g(X_{j+1})-g(X_j)\mid X_j } \ge -\card{s}K/2$, and,
with probability at least~$1/2$, $\sum_{j\in B} g(X_{j+1})-g(X_j)  \ge -\card{s}K$
using Markov's inequality (noting that all terms are non-positive, so Markov's inequality can be applied
to the negative of the random variable). We assume this to happen, which accounts for the  factor
$1/2$ in the statement of the lemma. Thus, by combining the effect of the
rw-steps with the b-steps
we obtain $g(X_{t_i})-g(X_0)\ge 3\sqrt{\card{s}}K-\card{s}K \ge 3\sqrt{\card{s}}K-\sqrt{\card{s}}K \ge 2\sqrt{\card{s}}K$.

Finally, from \eqref{eq:pot-stretch}, we know
that $g(X_{t_i})-g(X_0)\ge 2\sqrt{\card{s}}K $ implies that $X_{t_i}-X_0\le sK<0$, hence
clearly $\sum_{j=0}^i (X_{t_j+1}-X_{t_j})\le sK$. By \eqref{eq:prob-normalized-cga} and Lemma~\ref{lem:bound-cdf-normal} (in the appendix),
to bound $\prob{Y_{t_i}\ge 3\sqrt{\card{s}}K}$ from below, we compute
\[
\frac{3\sqrt{\card{s}}}{\card{s}\sqrt{\alpha/4}}+3.1\card{s}\sqrt{\alpha}
\le \frac{13}{\sqrt{\card{s}\alpha}}
\]
using $\card{s}\le 1$ and $\alpha\le 1$,
and get
\[
\left(\frac{1}{13\sqrt{1/(\card{s}\alpha)}}-\frac{1}{(13\sqrt{1/(\card{s}\alpha)})^{3}}\right)\frac{1}{\sqrt{2\pi}}e^{-169/(2\card{s}\alpha)}
=: p(\alpha,s).
\]
This means
 that distance $sK$ (in negative direction) is bridged by the rw-steps before or at time $t_i$, where $i=\alpha s^2K^2$,
 with probability at least~$(1/2-o(1))p(\alpha,s)$, where the factor~$1/2$ comes from the application
of Markov's inequality. Undoing the scaling of the state space,
this corresponds to an accumulated change of the actual state
of \cga in rw-steps by $s$; more formally, $\left(\sum_{i=0}^{t} \Delta_{i}\right) \le s$
in terms of the original state space.
%Note that the very same analysis (modulo
%flipping signs, \eg, in the application of Lemma~\ref{lem:bound-cdf-normal})
%applies if~$s>0$ and the probability of $Y_{t_i}\le -sK$ is
%to be analyzed; in fact, b-steps are only helpful in this case.
This establishes also the first statement of the lemma and completes the proof.
\end{proofof}

Lemma~\ref{lem:distribution-cga-self-loop} requires a bounded number of b-steps. To establish this, we first show that, during the early stages of a run, the probability of a b-step is only $O(1/\sqrt{n})$. Intuitively, during early stages of the run many bits will have marginal probabilities in the interval $[1/6, 5/6]$. Then the standard sampling deviation of the \onemax-value is of order $\Theta(\sqrt{n})$, and the probability of a b-step is $1-\prob{R_t} = O(1/\sqrt{n})$.
The link between $1-\prob{R_t}$ and the standard deviation already appeared
in Lemma~\ref{lem:drift-for-bit-in-cga} above; roughly, it says that every step is a b-step for bit~$i$ with probability at least
$(\sum_{j \neq i} p_j(1-p_j))^{-1/2}$, which is the reciprocal of the standard deviation in terms
of the other bits.

The following lemmas, most notably Lemma~\ref{lem:Poisson-mode} and Lemma~\ref{lem:prob-of-Bernoulli-step},
represent a kind of counterpart of Lemma~\ref{lem:drift-for-bit-in-cga}, but here we seek an \emph{upper} bound on $1-\prob{R_t}$.
The analysis in Lemma~\ref{lem:Poisson-mode} is non-trivial and uses advanced lemmas on properties of
the binomial distribution, including Schur-convexity. Lemma~\ref{lem:prob-of-Bernoulli-step} then applies the general
 Lemma~\ref{lem:Poisson-mode} to bound the probability of a b-step.

\begin{lemma}
\label{lem:Poisson-mode}
Let $S$ be the sum of $m$ independent Poisson trials with probabilities $p_1, \dots, p_m$ such that $1/6 \le p_i \le 5/6$ for all~$1 \le i \le m$. Then we have that for all $0 \le s \le m$,
\[
\Pr(S = s) = O(1/\sqrt{m}).
\]
\end{lemma}

\begin{proof}
%{Lemma~\ref{lem:Poisson-mode}}
Samuels~\cite{Samuels1965} showed that $\Pr(S = s)$ is nondecreasing in $s \le \E{S}$ and nonincreasing in $s \ge \E{S}$, hence it is maximal for $s \in \{\lfloor \E{S}\rfloor, \lceil \E{S}\rceil\}$.
Hence for every $0 \le s \le m$,
\begin{align*}
\Pr(S = s) \le\;& \max\{\Pr(S = \lfloor \E{S}\rfloor), \Pr(S = \lceil \E{S}\rceil)\}\\
\le\;& \Pr(\lfloor \E{S} \rfloor - 2 \le S \le \lceil \E{S} \rceil + 1).
\end{align*}
As remarked in~\cite[page~496]{MarshallInequalities}, the above is Schur-convex in $p_1, \dots, p_m$; this statement goes back to Gleser~\cite{Gleser1975}.
Hence the above probability is maximized if the vector of probabilities $(p_1,\dots,p_m)$ is a maximal element w.\,r.\,t.\ the preorder of majorization, i.\,e.\ for a fixed sum of $\E{S} = \sum_{i=1}^m p_i$, all probabilities $p_i$ are at their respective borders: $p_i \in \{1/6, 5/6\}$, except for potentially one probability.

Let $p_1', \dots, p_m'$ denote such a best case distribution, i.\,e. for $S' := \sum_{i=1}^m p_i'$ we have $\E{S'} = \E{S}$ and $p_1' = \dots = p_k' = 1/6$, $p_{k+1}' = \dots = p_{m-1}' = 5/6$ and $p_m' \in [1/6, 5/6]$.
Then
\begin{align*}
\Pr(S = s) \le\;& \Pr(\lfloor \E{S} \rfloor - 2 \le S \le \lceil \E{S} \rceil + 1)\\
\le\;& \Pr(\lfloor \E{S'} \rfloor - 2 \le S' \le \lceil \E{S'} \rceil + 1)
\end{align*}
Now assume that $k \ge (m-1)/2$. We apply the principle of deferred decisions and assume that the values of bits $k+1, \dots, m$ are known. Let $Y_k := \sum_{i=1}^k y_i$ be the number of ones on the $k$ bits having a marginal probability of $1/6$. Note that there is at most one value of $Y_k$ that leads to a particular value of $S'$. Bounding all such probabilities by the probability of the mode, $\max_y \Pr(Y_k = y)$, we get
\[
\Pr(\lfloor \E{S'} \rfloor - 2 \le S' \le \lceil \E{S'} \rceil + 1)
\le 5 \max_{y} \Pr(Y_k = y).
\]
Since $Y_k$ follows a Binomial distribution with parameters $k$ and $1/6$, its mode is either $\lfloor k/6 \rfloor$ or $\lceil k/6 \rceil$.
Using bounds on binomial coefficients (Corollary~2.3 in \cite{Stanica2001} for $m=1$), it is easy to show (see Lemma~\ref{lem:mode-of-Binomial-dist} below) that $\Pr(Y_k = \lfloor k/6 \rfloor)$ and $\Pr(Y_k = \lceil k/6 \rceil)$ are both bounded by $O(1/\sqrt{k})$, hence
$\max_y \Pr(Y_k = y) = O(1/\sqrt{k}) = O(1/\sqrt{n})$, and
\[
\Pr(S = s) \le 5 \max_{y} \Pr(Y_k = y) = O(1/\sqrt{n}).
\]
The case $k < (m-1)/2$ is symmetric; we then consider the Binomial distribution over bits $k+1, \dots, m-1$, of which there are at least $m-1-k \ge (m-1)/2$ many.
\end{proof}

The following lemma is used to bound the mode of the binomial distribution in the proof of Lemma~\ref{lem:Poisson-mode} above.
\begin{lemma}
\label{lem:mode-of-Binomial-dist}
Let $X\sim \mathrm{Bin}(n,p)$ for some $0<p<1$. If $np$ is an integer then
\[
\prob{X=np} \le \frac{1}{\sqrt{2\pi np(1-p)} }.
\]
Otherwise,
\[
\prob{X=\lceil np\rceil} \le \frac{e}{\sqrt{2\pi np(1+a)(1-p(1+a))} }
\]
for $a = \frac{\lceil np\rceil}{np}-1 \le 1/(np)$, and
\[
\prob{X=\lfloor np\rfloor} \le \frac{e}{\sqrt{2\pi np(1-a)(1-p(1-a))} }
\]
for $a = 1-\frac{\lfloor np\rfloor}{np} \le 1/(np)$.
\end{lemma}

\begin{proof}
We start with the integral case.
By definition,
\[
\prob{X=k} = \binom{n}{k} p^{k} (1-p)^{n-k}
\]
We use the following bound on the binomial coefficient (Corollary~2.3 in \cite{Stanica2001} for $m=1$):
\[
\binom{n}{\alpha n} < \frac{1}{\sqrt{2\pi n \alpha(1-\alpha)} \alpha^{\alpha n} (1-\alpha)^{n-\alpha n}}.
\]
Plugging this in the formula for $\prob{X=k}$ with $k=np$ and $\alpha=p$, we get the desired result if $np$ is integer.

If $np$ is not an integer, then we write $\lceil np\rceil = np(1+a)$ for some $a\le 1/(np)$. Following the
same approach with $k=np(1+a)$ and $\alpha = p(1+a)$, we get
\begin{align*}
& \prob{X=\lceil np\rceil}\\
& \le \frac{p^{np(1+a)} (1-p)^{n-np(1+a)}}{\sqrt{2\pi n p(1+a)(1-p(1+a))} ((1+a)p)^{np(1+a)} (1-p(1+a))^{n-np(1+a)}} \\
& \le
 \frac{1}{\sqrt{2\pi n p(1+a)(1-p(1+a))}} \left(1+\frac{pa}{1-p(1+a)}\right)^{n-np(1+a)} \\
& \le \frac{1}{\sqrt{2\pi n p(1+a)(1-p(1+a))}} e^{\frac{pa}{1-p(1+a)} (n-np(1+a))} \le
\frac{e}{\sqrt{2\pi n p(1+a)(1-p(1+a))}},
\end{align*}
where the second inequality bounded $(p/(p(1+a)))^{np(1+a)}\le 1$, the third used $1+x\le e^x$ and the fourth
$pan\le 1$. The bound for $\prob{X=\lfloor np\rfloor}$ is proved analogously, with $1-a$ taking the role of $1+a$ and
the roles of $p$ and $1-p$ swapped.
\end{proof}

We remark that if $p=1/2$ and $np$ integer, we recover from the previous lemma the following well-known bound on the central binomial coefficient: $\binom{n}{n/2} \le  2^{n+1/2}/(\sqrt{\pi n})$.

\begin{lemma}
\label{lem:prob-of-Bernoulli-step}
Assume that at time~$t$ there are $\gamma n$ bits for some constant~$\gamma > 0$ bits whose marginal probabilities are within $[1/6, 5/6]$. Then the probability of having a b-step on any fixed bit position is
\[
1 - \prob{R_t} = O(1/\sqrt{n}),
\]
regardless of the decisions made in this step on all other $n-\gamma n-1$ bits.
\end{lemma}

\begin{proof}
%{Lemma~\ref{lem:prob-of-Bernoulli-step}}
We know from our earlier discussion that a b-step at bit~$i$ requires $D_t \in \{-1, 0\}$ where $D_t := |x| - |x_{i}| - (|y| - |y_{i}|)$ is the change of the \onemax-value at bits other than~$i$ in the two solutions~$x$ and~$y$ sampled at time~$t$.

We apply the principle of deferred decisions and fix all decisions for creating $x$ as well as decisions for $y$ on all but the $m := \gamma n$ selected bits with marginal probabilities in $[1/6, 5/6]$. Let $p_1, p_2, \dots, p_{m}$ denote the corresponding marginal probabilities after renumbering these bits, and let $S$ denote the random number of these bits set to~1. Note that there are at most $2$ values for $S$ which lead to the algorithm making a b-step.

Since $S$ is determined by a Poisson trial with success probabilities $p_1, \dots, p_{m}$, Lemma~\ref{lem:Poisson-mode} implies that the probability of $S$ attaining any particular value is $O(1/\sqrt{m}) = O(1/\sqrt{n})$. Taking the union bound over 2 values proves the claim.
\end{proof}

%\begin{lemma}
%\label{lem:displacement-b-steps}
%Assume that during the first $\kappa K\sqrt{n}$ steps, $\kappa > 0$ a small enough constant, there are $c_1n$ bits for some constant~$c_1 > 0$ bits whose success probabilities are within $[1/6, 5/6]$. Then for any $t \le \kappa K\sqrt{n}$, the probability that all Bernoulli steps made during the first $t$ steps lead to a total change of its success probability outside of $[-1/6, 1/6]$ is at most $\kappa K\sqrt{n} \cdot e^{-\Omega(K)}$.
%\end{lemma}
%\begin{proof}
%According to Lemma~\ref{lem:prob-of-Bernoulli-step}, we know that the probability of a Bernoulli step is at most $c_2/\sqrt{n}$ for a positive constant~$c_2$. Hence the expected number of Bernoulli steps is at most $\kappa \cdot c_2 K$. Each Bernoulli step changes the success probability of a bit by $1/K$. A necessary condition for increasing the success probability by a total of at least $1/6$ is that we have at least $K/6$ Bernoulli steps amongst the first $t$ steps. Choosing $c$ small enough to make $\kappa \cdot c_2 K \le 1/2 \cdot K/6$, by Chernoff bounds the probability to get at least $K/6$ Bernoulli steps in $t \le \kappa K\sqrt{n}$ steps is $e^{-\Omega(K)}$.
%Taking the union bound over all $t \le \kappa K\sqrt{n}$ gives a probability bound of $\kappa K\sqrt{n} \cdot e^{-\Omega(K)}$.
%\end{proof}

Even though one main aim is to show that rw-steps make certain marginal probabilities reach their lower border, we will also
ensure that with high probability, $\Omega(n)$ marginal probabilities do not move by too much, resulting in a large sampling variance and a small
probability of b-steps. The following lemma serves this purpose. Its proof is a straightforward application of Hoeffding's
inequality since it is pessimistic here to ignore the self-loops.
\begin{lemma}
\label{lem:displacement-rw}
For any bit, with probability $\Omega(1)$ for any $t \le \kappa K^2$, $\kappa > 0$ a small enough constant, the first $t$ rw-steps lead to a total change of the bit's marginal probability within $[-1/6, 1/6]$. This fact holds independently of all other bits.

The probability that the above holds for less than $\gamma n$ bits amongst the first $n/2$ bits is $2^{-\Omega(n)}$, regardless of the decisions made on the last $n/2$ bits.
\end{lemma}

\begin{proof}
%{Lemma~\ref{lem:displacement-rw}}
Note that the probability of exceeding $[-1/6, 1/6]$ increases with the number of rw-steps that do increase or decrease the marginal probability (as opposed to self-loops). We call these steps \emph{relevant} and pessimistically assume that all $t$ steps are relevant steps.

Now defining $X_j := \sum_{i=1}^{j} X_i$ as the total progress in the first $j$ relevant steps, we have $\E{X_j} = 0$, for all $j \le t$, and the total change in these $j$ steps exceeds $1/6$ only if $X_j \ge K/6$. Applying a Hoeffding bound, Theorem~1.13 in~\cite{DoerrTools},
the maximum total progress is bounded as follows:
\[
\Pr\left(\max_{j \le t} X_j \le K/6\right) \le \exp\left(\frac{-2(K/6)^2}{4t}\right) \le \exp\left(-\frac{1}{12\kappa}\right).
\]
By symmetry, the same holds for the total change reaching values less or equal to~$-1/6$. By the union bound, the probability that the total change
always remains within the interval $[-1/6, 1/6]$ is thus at least
\[
1 - 2\exp\left(-\frac{1}{12\kappa}\right).
\]
Assuming $\kappa < 1/(12 \ln 2)$ gives a lower bound of~$\Omega(1)$.

Note that due to our pessimistic assumption of all steps being relevant, all bits are treated independently. Hence we may apply standard Chernoff bounds to derive the second claim.
\end{proof}

%
%These 3 lemmas taken together establish, for $K \le \sqrt{n}$, that $c_1n$ bits stay within success probabilities $[1/6, 5/6]$ over the first $\kappa K^2$ steps (as $\kappa K^2 \le \kappa \sqrt{n}K$).
%\begin{lemma}
%\label{lem:probabilities-stay-in-middle}
%Assume $K \le \sqrt{n}$. With probability $1-e^{-\Omega(\sqrt{n})}$, for a constant $c_1 > 0$, there are at least $c_1n$ bits whose success probabilities are within $[1/6, 5/6]$ at any step $t$ for all $t \le \kappa K^2$.
%\end{lemma}

The following lemma shows that whenever a small number of bits has reached the lower border for marginal probabilities, the remaining optimization time is $\Omega(n \log n)$ with high probability. The proof is similar to the well known coupon collector's theorem~\cite{Motwani1995}.

\begin{lemma}
\label{lem:after-hitting-lower-border}
Assume \cga reaches a situation where at least $\Omega(n^\epsilon)$ marginal probabilities attain the lower border~$1/n$. Then with probability $1 - e^{-\Omega(n^{\epsilon/2})}$, and in expectation, the remaining optimization time is $\Omega(n \log n)$.
\end{lemma}

\begin{proof}
%{Lemma~\ref{lem:after-hitting-lower-border}}
Let $m$ be the number of bits that have reached the lower border~$1/n$. A necessary condition for reaching the optimum within $t := (n/2-1)\cdot (\epsilon/2) \ln n$ iterations is that during this time each of these $m$ bits is sampled at value~1 in at least one of the two search points constructed. The probability that one bit never samples a~1 in $t$ iterations is at least $(1 - 2/n)^t$. The probability that all $m$ bits sample a 1 during $t$ steps is at most, using $(1-2/n)^{n/2-1} \ge 1/e$ and $1+x \le e^x$ for $x \in \R$,
\[
\left(1 - \left(1 - \frac{2}{n}\right)^t\right)^m
%\le \left(1 - e^{-c \ln n}\right)^m
\le \left(1 - n^{-\epsilon/2}\right)^m
%\le \left(\exp\left(-n^{-\epsilon/2}\right)\right)^m
\le \exp(-\Omega(n^{\epsilon/2})).
\]
Hence with probability $1 - \exp(-\Omega(n^{\epsilon/2}))$ the remaining optimization time is at least $t = \Omega(n \log n)$. As $1 - \exp(-\Omega(n^{\epsilon/2})) = \Omega(1)$, the expected remaining optimization time is of the same order.
\end{proof}

We have collected most of the machinery to prove Theorem~\ref{theo:lower-cga}. The following lemma identifies a set of bits
that stay centered in a phase of $\Theta(K\min\{K,\sqrt{n}\})$ steps, resulting in a low probability of b-steps. Basically, the
idea is to bound the accumulated effect of b-steps in the phase using Chernoff bounds: given $K/6$ b-steps, a marginal
probability cannot change by more than $1/6$. Note that this applies to many, but not all bits. Later, we
will see that within the phase, some of the remaining bits will reach their lower border with not too low probability.

\begin{lemma}
\label{lem:prob-first-nhalf}
Let $\kappa > 0$ be a small constant. There exists a constant $\gamma$, depending on~$\kappa$, and a selection $S$ of $\gamma n$ bits among the first $n/2$ bits such that the following properties hold regardless of the last $n/2$ bits throughout the first $T := \kappa K \cdot \min\{K, \sqrt{n}\}$ steps of \cga with $K \le \poly(n)$, with probability $\poly(n) \cdot 2^{-\Omega(\min\{K, n\})}$:
%such that with probability $1-2^{-\Omega(K)}$, the following properties holds throughout the first $T := \kappa K \cdot \min\{K, \sqrt{n}\}$ steps of the \cga, regardless of the decisions made on the last $n/2$ bits:
\begin{enumerate}
\item the marginal probabilities of all bits in $S$ is always within $[1/6, 5/6]$ during the first $T$ steps,
\item the probability of a b-step at any bit is always $O(1/\sqrt{n})$ during the first $T$ steps, and
\item the total number of b-steps for each bit is bounded by $K/6$, leading to a displacement of at most~$1/6$.
\end{enumerate}
\end{lemma}

\begin{proof}
%{Lemma~\ref{lem:prob-first-nhalf}}
The first property is trivially true at initialization, and we show that an event of exponentially small probability needs to occur in order to violate the property. Taking a union bound over all $T$ steps ensures that the property holds throughout the whole phase of $T$ steps with the claimed probability.

By Lemma~\ref{lem:displacement-rw}, with probability $1-2^{-\Omega(n)}$, for at least $\gamma n$ of these bits the total effect of all rw-steps is always within $[-1/6, +1/6]$ during the first $T \le \kappa K^2$ steps. We assume in the following that this happens and take $S$ as a set containing exactly $\gamma n$ of these bits.

It remains to show that for all bits in~$S$ the total effect of b-steps is bounded by $1/6$ with high probability.
Note that, while this is the case, according to Lemma~\ref{lem:prob-of-Bernoulli-step}, the probability of a b-step at every bit in~$S$ is at most $c_2/\sqrt{n}$ for a positive constant~$c_2$. This corresponds to the second property, and so long as this holds, the expected number of b-steps in $T \le \kappa K^2$ steps is at most $\kappa \cdot c_2 K$. Each b-step changes the marginal probability of the bit by $1/K$. A necessary condition for increasing the marginal probability by a total of at least $1/6$ is that we have at least $K/6$ b-steps amongst the first $T$ steps. Choosing $\kappa$ small enough to make $\kappa \cdot c_2 K \le 1/2 \cdot K/6$, by Chernoff bounds the probability to get at least $K/6$ b-steps in $T$ steps is $e^{-\Omega(K)}$.
%Hence the third property holds with probability $1-e^{-\Omega(K)}$.
In order for the first property to be violated, an event of probability $e^{-\Omega(K)}$ is necessary for any bit in~$S$ and any length of time $t \le T$; otherwise all properties hold true.

Taking the union bound over all $T \le \kappa K^2$ steps and all $\gamma n$ bits gives a probability bound of $\kappa K^2 \cdot \gamma n \cdot e^{-\Omega(K)} \le \poly(n) \cdot 2^{-\Omega(K)}$ for a property being violated. This proves the claim.
%Hence, in order for one of these $\gamma n$ success probabilities to leave the interval $[1/6, 5/6]$, an event of probability at most $\kappa K^2 \cdot e^{-\Omega(K)}$ is necessary. Taking the union bound over $\gamma n$ bits leads to a failure probability of $\kappa K^2 n \cdot e^{-\Omega(K)} = 2^{-\Omega(K)}$.
%
%Assuming that no failure happens, we always have at least $\gamma n$ bits in the interval $[1/6, 5/6]$ for the first $T$ steps. This proves the claim and ensures that for \emph{all} bits, the gain from b-steps is at most $1/6$ during the first $T$ steps.
\end{proof}

Finally, we put everything together to prove our lower bound for \cga.

\begin{proofof}{Theorem~\ref{theo:lower-cga}}
If $K = O(1)$ then it is easy to show, similarly to Lemma~\ref{lem:mmas-less-than-rho}, that each bit independently hits the lower border with probability $\Omega(1)$ by sampling only zeros. Then the result follows easily from Chernoff bounds and Lemma~\ref{lem:after-hitting-lower-border}. Hence we assume in the following $K = \omega(1)$.

For $K \ge \sqrt{n}$, Lemma~\ref{lem:prob-first-nhalf} implies a lower bound of $\Omega(K \sqrt{n})$ as then the probability of sampling the optimum in any of the first $T := \kappa K \cdot \min\{K, \sqrt{n}\}$
steps is at most $(5/6)^{\gamma n} = 2^{-\Omega(n)}$. Taking a union bound over the first $T$ steps and adding the error probability from Lemma~\ref{lem:prob-first-nhalf} proves the claim for a lower bound of $\Omega(K \sqrt{n})$ with the claimed probability.
This proves the theorem for $K = \Omega(\sqrt{n}\log n)$ as then the $\Omega(\sqrt{n}K)$ term dominates the runtime. Hence we may assume $K = o(\sqrt{n}\log n)$ in the following and note that in this realm proving a lower bound of $\Omega(n \log n)$ is sufficient as here this term dominates the runtime.

We still assume that the events from Lemma~\ref{lem:prob-first-nhalf} apply to the first $n/2$ bits.
We now use
Lemma~\ref{lem:distribution-cga-self-loop} to show that some marginal probabilities amongst the last $n/2$ bits  are likely to walk down to the lower border.
Note that Lemma~\ref{lem:distribution-cga-self-loop} applies for an arbitrary (even adversarial) mixture of
rw-steps and b-steps over time, so long as the overall number of b-steps is bounded. This allows us to regard the progress in
rw-steps as independent between bits.

In more detail, we will apply both statements of Lemma~\ref{lem:distribution-cga-self-loop}
to a fresh marginal probability from the last $n/2$ bits, to prove that
it walks to its lower border with a not too small probability.
First we apply the second statement
of the lemma for a positive displacement of $s:=1/6$ within  $T$ steps, using $\alpha := T/((sK)^2)$. The random variable $T_s$ describes
the first point of time where the  marginal probability reaches a value of at least $1/2+1/6+s=5/6$ through a mixture of
b- and rw-steps. This holds since we
work under the assumption that
the b-steps only account for a total displacement of at most~$1/6$ during the phase. Lemma~\ref{lem:distribution-cga-self-loop}
now gives us a probability of at least
$1-e^{-1/(4\alpha)} = \Omega(1)$ (using $\alpha=O( 1)$) for the event that the marginal probability does not exceed $5/6$. In the following,
we condition on this event.

We then revisit the same stochastic process and apply Lemma~\ref{lem:distribution-cga-self-loop} again to show that, under this condition, the random walk achieves a negative displacement. Note that the event of not exceeding a certain positive displacement is positively correlated with the event of reaching a given negative displacement (formally,
the state of the conditioned stochastic process is always stochastically smaller than of the unconditioned process), allowing us to apply Lemma~\ref{lem:distribution-cga-self-loop} again despite dependencies between the two applications.

%Note that the condition formally assigns probability~$0$ to walks
%where the marginal probability exceeds the limit $5/6$. This only increases the probability of the marginal probability
%reaching values less than this limit.
We can therefore apply the first statement of Lemma~\ref{lem:distribution-cga-self-loop} for a negative
displacement of $s := -5/6$ within $T$ steps, still using $\alpha := T/((sK)^2)$.
Note that by Lemma~\ref{lem:prob-first-nhalf} at most $K/6 \le |s|K/4$ steps are b-steps.
The conditions on $\alpha$ hold as $0 < \alpha < 1$ choosing $\kappa$ small enough, and $1/\alpha = O(K/\min\{\sqrt{n}, K\}) = o(K)$ for $K = \omega(1)$. Also note that $1/\alpha = O(K/\min\{\sqrt{n}, K\}) = o(\log n)$ since $K = o(n \log n)$.
Now Lemma~\ref{lem:distribution-cga-self-loop}
states that the probability of the random walk reaching a total displacement of $-5/6$ (or hitting the lower
border before) is at least
\begin{align*}
& \Bigl(\frac{1}{2}-o(1)\Bigr)
\Bigl(\frac{1}{13\sqrt{1/(\card{s} \alpha)}}-\frac{1}{(13\sqrt{1 /(\card{s}\alpha)})^{3}}\Bigr)\frac{1}{\sqrt{2\pi}}e^{-\frac{169}{2\card{s}\alpha}}\\
%
% \left(\frac{1}{\sqrt{2s/\alpha}}-\frac{1}{(2s/\alpha)^{3/2}}\right)\frac{1}{\sqrt{2\pi}}e^{-s/\alpha}\\
%&\;\ge\; (1/2-o(1))
% \left(\frac{\sqrt{T}}{8\sqrt{3/\card{s}}sK}-\frac{T^{3/2}}{8^3 (3s)^{3/2}K^3}\right)\frac{1}{\sqrt{2\pi}}e^{-4/(\card{s}\alpha)}\\
%&\;\ge\; (1-o(1))
% \left(\frac{\sqrt{T}}{\sqrt{6s}sK}-\frac{T^{3/2}}{(6s)^{3/2}(sK)^3}\right)\frac{1}{\sqrt{2\pi}}e^{-4/(\card{s}\alpha)}\\
& \;=\; \Omega\!\left(\frac{1}{o(\sqrt{\log n})} \cdot e^{-o(\ln n)}\right) \;\ge\; n^{-\beta}
\end{align*}
for some $\beta = o(1)$. Combining with the probability of not exceeding $5/6$, the probability of
the bit's marginal probability  hitting the lower border within $T$ steps is $\Omega(n^{-\beta})$.
%Due to the symmetry of the variables $F_t$, the probability that a displacement of $-1/6$ is reached, given that either $+1/6$ or $-1/6$ is reached, is~$1/2$.
%By Chernoff bounds, with probability $1-2^{-\Omega(n^{1-\beta})}$ this happens to at least $n^{1-\beta}/2$ bits.
%Let us call $S_{\mathrm{low}}$ the set of these bits.
%For each of these bits, we apply Lemma~\ref{lem:distribution-cga-self-loop} again with $s := -1/2$, from the point in time where they first reach a displacement of $-1/6$ from rw-steps. The probability is again $n^{-\beta'}$ for some $\beta' = o(1)$.
%Note that here we eventually hit probability borders. But since the total displacement from b-steps is at most $1/6$ and earlier rw-steps have already led to a %displacement below $-1/6$, the probability of hitting the lower border is at least $1/2$ by symmetry of $F_t$ variables.
Hence by Chernoff bounds, with probability $1-2^{-\Omega(n^{1-\beta})}$,
the final number of bits hitting the lower border within $T$ steps is $\Omega(n^{1-\beta}) = \Omega(n^{1-o(1)})$.

Once a bit has reached the lower border, while the probability of a b-step is $O(1/\sqrt{n})$, the probability of leaving the bound again is $O(n^{-3/2})$ as it is necessary that either the bit is sampled as~1 at one of the offspring and a b-step happens, or in both offspring the bit is sampled at~1. So the probability that this does not happen until the $T = O(n \log n)$ steps are completed is $(1-O(n^{-3/2}))^{T} \le e^{-O(\log(n)/\sqrt{n})} = o(1)$. Again applying Chernoff bounds leaves $\Omega(n^{1-o(1)})$ bits at the lower border at time~$T$ with probability $1-2^{-\Omega(n^{1-o(1)})}$.

Then Lemma~\ref{lem:after-hitting-lower-border} implies a lower bound of $\Omega(n \log n)$ that holds with probability $1-2^{-\Omega(n^{1/2-o(1)})}$.
\end{proofof}

\subsection{Proof of Lower Bound for \twoant}
\label{sec:proof-lower-mmas}
We will use, to a vast extent, the same approach as in Section~\ref{sec:proof-lower-cga} to prove Theorem~\ref{theo:lower-mmas}.
 Most of the lemmas can be applied directly or with very minor changes. In particular,
Lemma~\ref{lem:displacement-rw}, Lemma~\ref{lem:after-hitting-lower-border} and Lemma~\ref{lem:prob-first-nhalf}
also apply to \twoant by identifying $1/K$ with~$\rho$. Intuitively, this holds since the analyses
of b-steps always pessimistically bound the  absolute change of a marginal probability
by the update strength ($1/K$ for \cga). This also holds with respect to the update strength
$\rho$ for \twoant.

To prove lower bounds on the time to hit a border through rw-steps,
the next lemma is used. It is very similar to Lemma~\ref{lem:distribution-cga-self-loop},
except for two minor differences: first, also the
accumulated effect of b-steps is included in the quantity
$p_t-p_0$ analyzed in the lemma. Second, considerations
are stopped when the marginal probability becomes less than~$\rho$ or more than $1-\rho$.
This has technical reasons but is not a crucial restriction.
We supply an additional lemma,  Lemma~\ref{lem:mmas-less-than-rho} below, that applies
when the marginal probability is less than~$\rho$. The latter lemma uses known analyses similar to so-called landslide sequences
defined in~\cite[Section~4]{Neumann2010a}.

\begin{lemma}
\label{lem:distribution-mmas-self-loop}
Consider a bit of \twoant on \OneMax and let $p_t$ be its marginal probability at time~$t$. We say
that the process breaks a border at time~$t$
if $\min\{p_t,1-p_t\}\le \max\{1/n,\rho\}$.
Given $s\in\R$  and arbitrary
starting state~$p_0$,
let $T_s$ be the smallest $t$ such that $\sgn(s)(p_t-p_0) \ge \card{s}$ holds or
a border is broken.

Choosing $0<\alpha<1$, where $1/\alpha=o(\rho^{-1})$, and
${-1 < s< 0}$ constant, and
assuming that
every step is a b-step with probability at most $\rho/(4\alpha)$,
we have
\begin{align*}
& \prob{T_s\le \alpha (s/\rho)^2    \text{ or $p_t$ exceeds $5/6$ before~$T_s$}}\\
& \qquad \ge (1-o(1))\cdot \Bigl(\frac{1}{\sqrt{(24/(\card{s}\alpha)}}-\frac{1}{(24/(\card{s}\alpha))^3}\Bigr)\frac{1}{\sqrt{2\pi}}e^{-288/(\card{s}\alpha)}.
\end{align*}

Moreover, for any $\alpha>0$ and constant $0<s<1$,
if there are at most $s/(2\alpha\rho)$ b-steps until time $\alpha (s/\rho)^2$,
then
\[\prob{T_s\ge \alpha (s/\rho)^2 \text{ or a border is broken until time $\alpha (s/\rho)^2$}}
\ge 1-e^{-1/(16\alpha)}.\]
\end{lemma}

\begin{proof}
%{Lemma~\ref{lem:distribution-mmas-self-loop}}
We follow similar ideas  as in the proof of Lemma~\ref{lem:distribution-cga-self-loop}.
Again, we start with the second statement, where $s\ge 0$ is assumed, and
aim for applying a Hoeffding bound.
 We note that a marginal probability
 of \twoant can only change by an absolute amount of at most~$\rho$
in a step. Hence, the b-steps until time $\alpha(s/\rho^2)$
account for an increase of the $X_t$-value by
at most $s/2$.   With respect to the rw-steps,
Theorem~1.11 from \cite{DoerrTools} can be applied with $c_i=2\rho$ and
$\lambda=s/2$.

Also for the first statement, we
follow the ideas from the proof of Lemma~\ref{lem:distribution-cga-self-loop}. In particular,
the borders stated in the lemma will be ignored as all considerations are stopped when
they are reached. We will apply a potential function and estimate
its first and second moment separately with respect to rw-steps and non-rw steps.

Our potential function is
\[
g(x):=\int_{x}^{1/2} \frac{1}{\rho \sqrt{z}}\,\mathrm{d}z,
\]
which can be considered the continuous analogue of the function $g$ used
in the proof of Lemma~\ref{lem:distribution-cga-self-loop}.
For $r>0$ and $x\le 1/2$, we note
that
\begin{equation}
g(x-r)-g(x) = \frac{2}{\rho} \bigl(\sqrt{x}-\sqrt{x-r}\bigr).
\label{eq:stretch-mmas}
\end{equation}

For better readability, we denote by $X_t:={p_{t}}$, $t\ge 0$, the stochastic process
obtained by listing the marginal probabilities of the considered bit over time.
Let $Y_t:=g(X_t)$ and $\Delta_t:=Y_{t+1}-Y_t$. In the remainder of this proof, we assume
$X_t\le 1/2$; analyses for the case $X_t>1/2$ are symmetrical by switching the sign of $\Delta_t$. We also assume $X_t \ge \rho$
as we are only interested in statements before the first point of time  where a border is broken.

We claim for all $t\ge 0$ where rw-steps occur (hence, formally
we enter the conditional probability space on $R_t$, the event that
an rw-step occurs at time~$t$)  that
\begin{align}
0 & \le \expect{\Delta_t\mid X_t; R_t} \le \frac{3\rho}{2\sqrt{X_t}} = o(1)\label{eq:mmas-expect-rw}\\
 & \quad \Var(\Delta_t \mid X_t; R_t)   \ge 1/16 \label{eq:mmas-var-rw}.
\end{align}
We start with the bounds on the expected value.
Note that by the properties of rw-steps for \twoant, where there
are two possible successor states, we get the martingale property
$\expect{X_{t+1}\mid X_t}=(1-X_t) (X_t-\rho X_t) + X_t (X_t+\rho (1-X_t) ) =X_t$. Since
$g(x)$ is a convex function on $[0,1/2]$, we have by Jensen's inequality
$\expect{\Delta_{t}\mid X_t} = \expect{g(X_{t+1})\mid X_t} - g(X_t) \ge g(\expect{X_{t+1}\mid X_t}) -g(X_t) = 0$. To bound
the expected value from above, we carefully estimate the error introduced by the convexity.
Note that \begin{equation}
g(x-x\rho)-g(x)=\int_{x-x\rho}^{x} \frac{1}{\rho \sqrt{z}}\,\mathrm{d}z \le
 \frac{x}{\sqrt{x-x\rho}}
\label{eq:mmas-change-g-1}
\end{equation}
 since the integrand is non-increasing. Analogously,
\begin{equation}
\frac{1-x}{\sqrt{x+(1-x)\rho}} \le g(x)-g(x+(1-x)\rho) \le \frac{1-x}{\sqrt{x}}
\label{eq:mmas-change-g-2}
\end{equation}
Inspecting the $g$-values of two possible successor states of $x:=X_t$, we get that
\begin{align}
&  \expect{\Delta_t\mid X_t=x} = \expect{g(X_{t+1}) -g(x) \mid X_t=x} \\
  & \le (1-x) \frac{x}{\sqrt{x-x\rho}} - x \frac{1-x}{\sqrt{x+(1-x)\rho}}
 = (1-x)x  \left(\frac{1}{\sqrt{x-x\rho}} - \frac{1}{\sqrt{x+(1-x)\rho}}\right)\notag\\
& = (1-x) x \cdot \frac{\sqrt{x+(1-x)\rho } -\sqrt{x-x\rho }}{\sqrt{x+(1-x)\rho } \cdot \sqrt{x-x\rho }}
 \le \frac{(1-x)x \frac{\rho}{2\sqrt{x-x\rho}}}{x-x\rho }
  \le \frac{x\rho}{2(x/2)^{3/2}} \notag\\
 & \le \frac{3\rho}{2\sqrt{x}} \label{eq:mmas-distr-upper-exp},
\end{align}
where the third-last inequality estimated $1-x\le 1$ and used that
$f(z+\rho)-f(z)\le \rho f'(z)$ for any concave, differentiable
function $f$ and $\rho\ge 0$; here using $f(z)=\sqrt{z}$ and $z=x-\rho$. The penultimate used $\rho\le 1/2$. Since the final
bound is $O(\rho/\sqrt{x}) = o(1)$ due to our assumption on $X_t\ge \rho$, we have proved  \eqref{eq:mmas-expect-rw}.

We proceed with the bound on the variance. Note that
\begin{align*} \Var(\Delta_t \mid X_t) &
\ge \expect{(\Delta_t - \expect{\Delta_{t}\mid X_t=x})^2 \cdot
\indic{\Delta_{t}\le 0}\mid X_t=x }\\
& \ge \expect{(\Delta_t )^2 \cdot
\indic{\Delta_{t}\le 0}\mid X_t=x }
\end{align*}
since $\expect{\Delta_t\mid X_t} \ge 0$. We note that
for $X_t=x$, we have $\prob{X_{t+1}\ge x} = x$. On
$X_{t+1}\ge x$, we have $\Delta_t<0$, which means $\prob{\Delta_t < 0} = x$.  Now,
$\card{\Delta_{t}} = g(x+(1-x)\rho) - g(x) \ge \frac{1-x}{\sqrt{x+\rho (1-x)}} \ge \frac{1-x}{\sqrt{x+x(1-x)}} \ge
\frac{1}{4\sqrt{x}}$, where the penultimate inequality used $\rho\le x$ and the last one
$x\le 1/2$. Plugging this in, we get
\begin{align*} \Var(\Delta_t \mid X_t=x) \ge x \cdot \left(\frac{1}{4\sqrt{x}}\right)^2 \ge \frac{1}{16},
\end{align*}
which completes the proof of \eqref{eq:mmas-var-rw} with respect to rw-steps.

We now consider the case that a b-step occurs at time~$t$. We are only interested in
bounding $\expect{\Delta_t\mid X_t}$ from below now. Given $X_t=x$, we have $X_{t+1}>x$ (which means
$\Delta_t<0$)
with probability at most $1-(1-x)^2 = 1-(1-2x+x^2) \le 2x$. With the remaining probability,
$X_{t+1}<x$. Since $X_{t+1}\le x+\rho$, we get
\begin{equation}
\expect{\Delta_t\mid X_t=x; \overline{R_t}} \ge -2x \int_{x}^{x+\rho} \frac{1}{\rho\sqrt{z}}\mathrm{d}z \ge -2\sqrt{x}.
\label{eq:mmas-expect-b}
\end{equation}
Now, since by assumption a b-step occurs with probability at most $\rho/(4\alpha)$, the unconditional
expected value of $\Delta_t$ can be computed using the superposition equality. Combining \eqref{eq:mmas-expect-rw}
and \eqref{eq:mmas-expect-b}, we get
\begin{equation}
\expect{\Delta_t\mid X_t=x} \ge 0 - \frac{\rho}{4\alpha} 2\sqrt{x} \ge -\frac{\rho}{2\alpha}.
\label{eq:mmas-expect-total}
\end{equation}
since $x\le 1$.
By the law of total probability, we get for the unconditional variance that
\[
\Var(\Delta_t\mid X_t) = \Var(\Delta_t\mid X_t; R_t)  \prob{R_t} +
\Var(\Delta_t\mid X_t; \overline{R_t}) (1- \prob{R_t}) ,
\]
 Since $\prob{R_t}\ge 1/2$,
we altogether have for the unconditional variance that
\[
\Var(\Delta_t\mid X_t=x) \ge 1/32.
\]

To apply the central limit theorem
 (Lemma~\ref{lem:weak-clt}) on the sum of the $\Delta_t$, we will verify
the Lyapunov condition for  $\delta=1$ (smaller values could be used but
do not give any benefit) and $t=\omega(1/\rho)$ (which, as $t=\alpha (s/\rho)^{2}$,
holds due to our assumptions $1/\alpha=o(\rho^{-1})$ and $\card{s}=\Omega(1)$).
We compute
\begin{align*}
& \expect{\lvert\Delta_t - \expect{\Delta_t\mid X_t}\rvert^{3}\mid X_t} \\
& \le
\prob{\Delta_t>0} \cdot (\Delta_t - \expect{\Delta_t\mid X_t})^3
+ \prob{\Delta_t< 0} \cdot (\card{\Delta_t} + \card{\expect{\Delta_t\mid X_t}})^3\\
& \le (1-x)
 \left(\frac{x}{\sqrt{x-x\rho}} \right)^3 + x
\cdot \left(\frac{1-x}{\sqrt{x}} + \frac{3\rho}{2\sqrt{x}}+ \frac{\rho}{2\alpha} \right)^3 ,
\end{align*}
where we again have used \eqref{eq:mmas-change-g-1} and the upper bound from \eqref{eq:mmas-change-g-2}
with respect to the two outcomes of $X_{t+1}$. Moreover, we have used  the bound
$\expect{\Delta_t\mid X_t}\ge 0$ in the first term and
 $\expect{\card{\Delta_t}\mid X_t} \le 3\rho/(2\sqrt{x})+\rho/(2\alpha)$ in the second term, which
is a crude combination of
\eqref{eq:mmas-distr-upper-exp} and \eqref{eq:mmas-expect-total}. As $\rho\le 1/2$ and $\rho\le x$ as well as $\alpha\ge \rho$, the
expected value satisfies
\begin{align*}
& \expect{\lvert\Delta_t - \expect{\Delta_t\mid X_t}\rvert^{3}\mid X_t} \le
\left(\frac{x}{\sqrt{x/2}}  \right)^3 + x \left(O\!\left(\frac{1}{\sqrt{x}}+3\sqrt{x} + \frac{1}{2}\right)^3\right) \\
& \quad \le 1 + x \left(O\!\left(\frac{1}{\sqrt{x}}\right)^3\right)
= O(1/\sqrt{x}) = O(1/\sqrt{\rho}),
\end{align*}
where we used $x\le 1$ and $x\ge \rho$. Using $s_t^2:=\sum_{j=0}^{t-1}\Var(\Delta_j\mid X_j)$ in the notation of Lemma~\ref{lem:weak-clt}
and using that $\Var(\Delta_j\mid X_j)\ge 1/32$, we get
\[
\frac{1}{s_t^{3}}\sum_{j=0}^{t-1} \expect{\lvert\Psi_j - \expect{\Psi_j}\rvert^{3}\mid X_j}
\le \frac{182}{t^{1.5}} O(t/\sqrt{\rho}) = O(\sqrt{1/(t\rho)}),
\]
which goes to $0$ as $t=\omega(1/\rho)$. This establishes
 the Lyapunov condition. Hence,
for the value $t:=\alpha (s/\rho)^2$ considered in the lemma,
we obtain that $\frac{Y_{t}-\expect{Y_{t}\mid X_0}}{s_t}$ converges in distribution
to the normal distribution $N(0,1)$. Note that $s_t^2 \ge \alpha (s/\rho)^2/32$ since $\Var(\Delta_t\mid X_t) \ge 1/32$. Hence,
$s_t = \sqrt{\alpha/32}(\card{s}/\rho)$, recalling that $s<0$. Moreover, as $x\le 5/6$ is assumed
in this part of the lemma,
by combining \eqref{eq:mmas-distr-upper-exp} and \eqref{eq:mmas-expect-total},
we get
$\expect{\Delta_t\mid X_t} \ge -\rho/(2\alpha)-\rho\cdot (3/2)\sqrt{6/5} \ge -\rho/(2\alpha) - 1.7\rho\ge -2.2\rho/\alpha$ and
 $\expect{Y_{t}}\ge t (-2.2\rho/\alpha)) \ge -2.2s^2/\rho$. Together, this means
$\frac{\expect{Y_t}}{s_t} \ge -\frac{2.2s^2/\rho}{\sqrt{\alpha/32}(\card{s}/\rho)} \ge -\sqrt{155/\alpha}\card{s}\ge -\sqrt{155/\alpha}$
since $\card{s}\le 1$ and $\alpha\le 1$.
By the normalization to $N(0,1)$, we have
 that \[
\prob{Y_{t}\ge  r} = \prob{\frac{Y_t}{s_t} - \frac{\expect{Y_{t}\mid X_0}}{s_t} \ge \frac{r}{s_t} - \frac{\expect{Y_{t}\mid X_0}}{s_t}},
\]
hence
$\prob{Y_{t}\ge  r}
\ge (1- o(1))(1-\Phi( r\rho/(\card{s}\sqrt{\alpha/32})+\sqrt{155/\alpha}))$
for any $r$ leading to a positive argument of~$\Phi$,
where $\Phi$ denotes the cumulative distribution function of the standard
normal distribution. We are interested in the event that $Y_{t}\ge 2\sqrt{\card{s}}/\rho$, recalling
that $s<0$ and $X_{t+1}\ge X_t \iff Y_{t+1}\le Y_t$.
% Since only $t/4$ b-steps occur until time~$t$ and each b-step can
%change the potential function by at most~$1$,
%we obtain $g(X_{t})-g(X_0)\le -3\sqrt{s}/\rho$.
We made this choice because the event
$Y_t = g(X_{t})-g(X_0)\ge 2\sqrt{\card{s}}/\rho $ implies that $X_{t}-X_0\le s$
by \eqref{eq:stretch-mmas}.

To compute the probability of the event $Y_t\ge 2\sqrt{\card{s}}/\rho$, we choose
$r=2\sqrt{\card{s}}/\rho$ and get
$r\rho/(\card{s}\sqrt{\alpha/32})+\sqrt{155/\alpha}) \le
 24/\sqrt{\card{s}\alpha}$. We get
\[
\prob{Y_{t}\ge  2\sqrt{\card{s}}/\rho}\ge (1- o(1))(1-\Phi(24/\sqrt{\card{s}\alpha})).\]

By Lemma~\ref{lem:bound-cdf-normal},
\[
1-\Phi(24/\sqrt{\card{s}\alpha}) \ge
\left(\frac{1}{24/\sqrt{\card{s}\alpha}}-\frac{1}{(24/\sqrt{\card{s}\alpha})^3}\right)\frac{1}{\sqrt{2\pi}}e^{-288/(\card{s}\alpha)} =: p(\alpha,s),
\]
which means
 that distance $s$ is bridged (in negative direction) before or at time $\alpha  (s/\rho)^2$
with probability at least~$(1-o(1))p(\alpha,s)$.
\end{proof}

The following lemma shows that a marginal probability of less than~$\rho$ is unlikely
to be increased again.

\begin{lemma}
\label{lem:mmas-less-than-rho}
In the setting of Lemma~\ref{lem:distribution-mmas-self-loop},
if $\min\{p_0,1-p_0\}\le \rho$, the marginal probability will reach the closer border from $\{1/n,1-1/n\}$
 in $O((\log n)/\rho)$ steps with probability at least $e^{-2/(1-e)}$. This even holds if each step is a b-step.
\end{lemma}

\begin{proof}
%{Lemma~\ref{lem:mmas-less-than-rho}}
We consider only the case $X_0\le \rho$ as the other case is symmetrical. The idea
is to consider $O(\log n)$ phases and prove that the $X_t$-value only decreases
throughout all phases with the stated probability. Phase~$i$, where $i\ge 0$, starts
at the first time where $X_t\le \rho e^{-i}$. Clearly, as $\rho\le 1$, at the latest in phase~$\ln n$
the border $1/n$ has been reached. We note that phase~$i$ ends after
$1/\rho$ steps if all these these steps decrease the value; here we use
that each step decreases by a relative amount of $1-\rho$ and that
 $(1-\rho)^{1/\rho}\le e^{-1}$.

The probability of decreasing the $X_t$-value in a step of phase~$i$ is at least
$(1-\rho e^{-i})^2 \ge 1-2e^{-i}\rho$ even if the step is a b-step. Hence, the probability
of all steps of phase~$i$ being decreasing is at least
$(1-2e^{-i}\rho)^{1/\rho} \ge e^{-2e^{-i}}$. For all phases together, the probability
of only having decreasing steps is still at least
\[
\prod_{i=0}^{\ln n} e^{-2e^{-i}} \ge e^{-2\sum_{i=0}^\infty e^{-i}} = e^{-2/(1-e)}
\]
as suggested.
\end{proof}

We have now collected all tools to prove the lower bound for \twoant.
\begin{proofof}{Theorem~\ref{theo:lower-mmas}}
This follows mostly the same structure as the proof of Theorem~\ref{theo:lower-cga}. Every occurrence
of the update strength $1/K$ should be replaced by~$\rho$. The analysis of b-steps is the same.

There is a minor change in the analysis of rw-steps. The two applications of Lemma~\ref{lem:distribution-cga-self-loop} are replaced
with Lemma~\ref{lem:distribution-mmas-self-loop}, followed by
an additional application of Lemma~\ref{lem:mmas-less-than-rho}. The slightly different
constants in the statement of Lemma~\ref{lem:distribution-cga-self-loop}
do not affect the asymptotic bound $\Omega(n^{-\beta})$ obtained.
%from applying Lemma~\ref{lem:distribution-cga-self-loop}, now Lemma~\ref{lem:distribution-mmas-self-loop}.
Neither does the additional application of Lemma~\ref{lem:mmas-less-than-rho}, which gives
a constant probability. We do not care about the time $O((\log n)/\rho)$ stated in Lemma~\ref{lem:mmas-less-than-rho}, since
we are only interested in a lower bound on the hitting time. Still, the assumptions on b-steps
 in Lemma~\ref{lem:distribution-mmas-self-loop}
differ slightly from the ones in Lemma~\ref{lem:distribution-cga-self-loop}. We have to verify these new assumptions.

%A minor change comes at the place where the proof of Theorem~\ref{theo:lower-cga} analyes
%the first $T/2$ steps and then
%argues with the symmetry of the $F_t$. This symmetry
%does not apply to \twoant since it in general does not take the
%step sizes in negative an positive direction. Instead, we apply the second statement of Lemma~\ref{lem:distribution-mmas-self-loop} to
%argue that $T_{+1/6}$ is bigger than $T/2$ with probability at least~$\Omega(1)$. Then, under this condition, the probability
%of $T_{-1/6}\le T/2$ is only increased. (If the state of \twoant is decreased, the probability of reaching the lower border
%in a given amount of time becomes only smaller.)

Lemma~\ref{lem:distribution-mmas-self-loop} requires in its first statement
that the probability of a b-step is at most $\rho/(4\alpha)$. Recall that
such a step has probability $O(1/\sqrt{n})$. We argue that
 $\rho/(4\alpha) \ge  c/\sqrt{n}$ for any constant~$c>0$ if $\kappa$ is small enough. To see this,
we simply recall that $\alpha=\kappa \sqrt{n}\rho/(3s^2)$ by definition and $\card{s}=\Omega(1)$.

Finally, the second statement
of Lemma~\ref{lem:distribution-mmas-self-loop} restricts the number of b-steps until time
$\alpha(s/\rho)^2$ to at most $s/(2\alpha\rho)$. Reusing that $\rho=O(\alpha/(\kappa\sqrt{n}))$, this
holds by Chernoff bounds with high probability if $\kappa$ is a sufficiently small constant. Hence,
the application of the lemma is possible.
%\item
%The last paragraph of the proof uses the update formula of \cga directly.
%However, we note that also for \twoant, the probability of leaving the
%bound again is $O(n^{-3/2 })$ as it is necessary that either the bit
%is sampled at $1$ in one of the two offspring and a b-step happens,
%or in both offspring the bit is sampled at $1$, which even has probability $O(n^{-2})$
%\end{itemize}
\end{proofof}

\section{Conclusions}
We have performed a runtime analysis of two probabilistic model-building GAs, namely
\cga and \twoant, on \onemax. The expected runtime of these algorithms was analyzed
in dependency of the so-called update strength $S=1/K$ and~$S=\rho$, respectively,
resulting
in the upper bound $O(\sqrt{n}/S)$ for  $S=O(1/\sqrt{n}\log n)$ and $\Omega(\sqrt{n}/S+n\log n)$. Hence,
$S\sim 1/\sqrt{n}\log n$ was identified as the choice for the update strength leading to asymptotically
smallest
expected runtime $\Theta(n\log n)$.

Our analyses of update strength reveal a general trade-off between the speed of learning and
genetic drift. High update strengths imply globally a fast adaptation of the probabilistic model
 but
impact the overall correctness of the model negatively, resulting
in increased risk of adapting to samples that are locally incorrect. We think that this
constitutes a universal limitation of the algorithms that extends to more general
classes of functions. As even on the simple \OneMax the update strength should not be bigger
than $1/(\sqrt{n}\log n)$, we propose this setting as a general rule of thumb.

Our analyses have developed a quite technical machinery for the analysis of genetic drift. These techniques
are not necessarily limited to \cga and \twoant on \OneMax.  We are optimistic
to be able to extend them other EDAs such as the UMDA
\cite{DangLehreGECCO15} and even classical GAs such as the simple GA \cite{OlivetoWittTCS15},
 where currently only
quite restricted lower bounds on the runtime are available.

\subsection*{Acknowledgements}
This research was initiated at Dagstuhl seminar 15211 ``Theory of Evolutionary Algorithms'' and also benefitted from Dagstuhl seminar 16011 ``Evolution and Computing''. The authors thank the organisers and participants of both seminars.
The research leading to these results has received funding from the European Union Seventh Framework Programme (FP7/2007-2013) under grant agreement no 618091 (SAGE) and from the Danish Research Council (DFF-FNU) under grant 4002-00542.
\bibliographystyle{abbrv}
\bibliography{eda-onemax}

\begin{thebibliography}{10}

\bibitem{Billingsley1995}
P.~Billingsley.
\newblock {\em Probability and measure}.
\newblock Wiley, 3rd edition, 1995.

\bibitem{ChenEtAlCEC2009}
T.~Chen, P.~K. Lehre, K.~Tang, and X.~Yao.
\newblock When is an estimation of distribution algorithm better than an
  evolutionary algorithm?
\newblock In {\em Proc.\ of CEC~'09}, pages 1470--1477. IEEE Press, 2009.

\bibitem{DangLehreGECCO15}
D.~Dang and P.~K. Lehre.
\newblock Simplified runtime analysis of estimation of distribution algorithms.
\newblock In {\em Proc.\ of GECCO~'15}, pages 513--518, 2015.

\bibitem{DoerrTools}
B.~Doerr.
\newblock Analyzing randomized search heuristics: Tools from probability
  theory.
\newblock In A.~Auger and B.~Doerr, editors, {\em Theory of randomized search
  heuristics}. World Scientific, 2011.

\bibitem{Doerr2010}
B.~Doerr, D.~Johannsen, and C.~Winzen.
\newblock Drift analysis and linear functions revisited.
\newblock In {\em Proc.\ of IEEE CEC~'10}, pages 1967--1974, 2010.

\bibitem{DoerrLenglerGECCO2015}
C.~Doerr and J.~Lengler.
\newblock {O}ne{M}ax in {B}lack-{B}ox {M}odels with {S}everal {R}estrictions.
\newblock In {\em Proc.\ of GECCO~'15}, pages 1431--1438. ACM Press, 2015.

\bibitem{Droste2006a}
S.~Droste.
\newblock A rigorous analysis of the compact genetic algorithm for linear
  functions.
\newblock {\em Natural Computing}, 5(3):257--283, 2006.

\bibitem{DJWBlackBox}
S.~Droste, T.~Jansen, and I.~Wegener.
\newblock Upper and lower bounds for randomized search heuristics in black-box
  optimization.
\newblock {\em Theory of Computing Systems}, 39, 2006.

\bibitem{Feller1}
W.~Feller.
\newblock {\em An Introduction to Probability Theory and Its Applications},
  volume~1.
\newblock Wiley, 1968.

\bibitem{FriedrichEtAlISAAC15}
T.~Friedrich, T.~K{\"{o}}tzing, M.~S. Krejca, and A.~M. Sutton.
\newblock The benefit of recombination in noisy evolutionary search.
\newblock In {\em Proc.\ of ISSAC~'15}, pages 140--150. Springer, 2015.

\bibitem{Gleser1975}
L.~J. Gleser.
\newblock On the distribution of the number of successes in independent trials.
\newblock {\em Ann. Probab.}, 3(1):182--188, 02 1975.

\bibitem{HarikEtAlCGA}
G.~R. Harik, F.~G. Lobo, and D.~E. Goldberg.
\newblock The compact genetic algorithm.
\newblock {\em IEEE Transactions on Evolutionary Computation}, 3(4):287--297,
  1999.

\bibitem{HauschildPelikan11}
M.~Hauschild and M.~Pelikan.
\newblock An introduction and survey of estimation of distribution algorithms.
\newblock {\em Swarm and Evolutionary Computation}, 1(3):111--128, 2011.

\bibitem{Johannsen2010}
D.~Johannsen.
\newblock {\em Random Combinatorial Structures and Randomized Search
  Heuristics}.
\newblock PhD thesis, Universit{\"a}t des Saarlandes, Saarbr{\"u}cken, Germany
  and the Max-Planck-Institut f{\"u}r Informatik, 2010.

\bibitem{MarshallInequalities}
A.~W. Marshall, I.~Olkin, and B.~C. Arnold.
\newblock {\em Inequalities: Theory of Majorization and Its Applications}.
\newblock Springer, 2nd edition, 2011.

\bibitem{Motwani1995}
R.~Motwani and P.~Raghavan.
\newblock {\em Randomized Algorithms}.
\newblock Cambridge University Press, 1995.

\bibitem{Neumann2009}
F.~Neumann, D.~Sudholt, and C.~Witt.
\newblock Analysis of different {MMAS} {ACO} algorithms on unimodal functions
  and plateaus.
\newblock {\em Swarm Intelligence}, 3(1):35--68, 2009.

\bibitem{Neumann2010a}
F.~Neumann, D.~Sudholt, and C.~Witt.
\newblock A few ants are enough: {ACO} with iteration-best update.
\newblock In {\em Proc.\ of GECCO~'10}, pages 63--70, 2010.

\bibitem{OlivetoWittTCS15}
P.~S. Oliveto and C.~Witt.
\newblock Improved time complexity analysis of the simple genetic algorithm.
\newblock {\em Theoretical Computer Science}, 605:21--41, 2015.

\bibitem{Rowe2013}
J.~E. Rowe and D.~Sudholt.
\newblock The choice of the offspring population size in the (1,$\lambda$)
  evolutionary algorithm.
\newblock {\em Theoretical Computer Science}, 545:20--38, 2014.

\bibitem{Samuels1965}
S.~M. Samuels.
\newblock On the number of successes in independent trials.
\newblock {\em Ann. Math. Statist.}, 36(4):1272--1278, 08 1965.

\bibitem{Stanica2001}
P.~St{\u a}nic{\u a}.
\newblock Good lower and upper bounds on binomial coefficients.
\newblock {\em Journal of Inequalities in Pure \& Applied Mathematics}, 2(3),
  2001.

\bibitem{Stutzle2000}
T.~St{\"u}tzle and H.~H. Hoos.
\newblock {MAX-MIN} ant system.
\newblock {\em Journal of Future Generation Computer Systems}, 16:889--914,
  2000.

\bibitem{Sudholt2012c}
D.~Sudholt.
\newblock A new method for lower bounds on the running time of evolutionary
  algorithms.
\newblock {\em IEEE Transactions on Evolutionary Computation}, 17(3):418--435,
  2013.

\bibitem{Witt2013}
C.~Witt.
\newblock Tight bounds on the optimization time of a randomized search
  heuristic on linear functions.
\newblock {\em Combinatorics, Probability and Computing}, 22:294--318, 2 2013.

\end{thebibliography}

%%%%%%%%%%%%%%%%%%%%%%%%%%%%%%%%%%%% APPENDIX %%%%%%%%%%%%%%%%%%%%%%%%%%%%%%%%%%%%
%%%%%%%%%%%%%%%%%%%%%%%%%%%%%%%%%%%% APPENDIX %%%%%%%%%%%%%%%%%%%%%%%%%%%%%%%%%%%%
%%%%%%%%%%%%%%%%%%%%%%%%%%%%%%%%%%%% APPENDIX %%%%%%%%%%%%%%%%%%%%%%%%%%%%%%%%%%%%
%%%%%%%%%%%%%%%%%%%%%%%%%%%%%%%%%%%% APPENDIX %%%%%%%%%%%%%%%%%%%%%%%%%%%%%%%%%%%%
%%%%%%%%%%%%%%%%%%%%%%%%%%%%%%%%%%%% APPENDIX %%%%%%%%%%%%%%%%%%%%%%%%%%%%%%%%%%%%
%%%%%%%%%%%%%%%%%%%%%%%%%%%%%%%%%%%% APPENDIX %%%%%%%%%%%%%%%%%%%%%%%%%%%%%%%%%%%%
%%%%%%%%%%%%%%%%%%%%%%%%%%%%%%%%%%%% APPENDIX %%%%%%%%%%%%%%%%%%%%%%%%%%%%%%%%%%%%
%%%%%%%%%%%%%%%%%%%%%%%%%%%%%%%%%%%% APPENDIX %%%%%%%%%%%%%%%%%%%%%%%%%%%%%%%%%%%%
%%%%%%%%%%%%%%%%%%%%%%%%%%%%%%%%%%%% APPENDIX %%%%%%%%%%%%%%%%%%%%%%%%%%%%%%%%%%%%
%%%%%%%%%%%%%%%%%%%%%%%%%%%%%%%%%%%% APPENDIX %%%%%%%%%%%%%%%%%%%%%%%%%%%%%%%%%%%%

\appendix
\section{General Tools}

%This appendix contains all proofs that were omitted from the main part. This is for the convenience of GECCO reviewers, allowing them to check our results for correctness at their discretion. The appendix will not be part of the final version.

\subsection{Generalized Variable Drift Theorem}

The following theorem is an easy generalization of~\cite[Theorem~1]{Rowe2013}.
\begin{theorem}[Generalized variable drift theorem]
\label{drift:johannsen-general}
Consider a stochastic process on $\N_0$. Suppose there is a monotonic increasing function $h: \R^+ \rightarrow \R^+$ such that the function $1/h(x)$ is integrable on $[1, m]$, and with
\[
	\Delta_k \geq h(k)
\]
for all $k \in \{1, \dots, m\}$. Then
%the expected first hitting time of state~0 is at most
%\[
%	\frac{1}{h(1)} + \int_1^m \frac{1}{h(x)} \;\mathrm{d}x.
%\]
%In addition,
the expected first hitting time of any state from $\{0, \dots, a-1\}$ for $a \in \N$ is at most
\[
	\frac{a}{h(a)} + \int_{a}^m \frac{1}{h(x)} \;\mathrm{d}x.
\]
\end{theorem}

\subsection{Bounds on the cumulative distribution function of the standard normal distribution}
To prove Lemmas~\ref{lem:distribution-cga-self-loop} and~\ref{lem:distribution-mmas-self-loop}, we
need the following estimates for $\Phi(x)$. More precise formulas are available (and can
be found by searching for bounds on the so-called error function), but are not required for our analysis.

\begin{lemma}[\cite{Feller1}, p. 175]
\label{lem:bound-cdf-normal}
For any $x>0$
\[
\left(\frac{1}{x}-\frac{1}{x^3}\right)\frac{1}{\sqrt{2\pi}}e^{-x^2/2}\;\le 1-\Phi(x) \;\le
\frac{1}{x}\frac{1}{\sqrt{2\pi}}e^{-x^2/2},
\]
and for $x<0$
\[
\left(\frac{-1}{x}-\frac{-1}{x^3}\right)\frac{1}{\sqrt{2\pi}}e^{-x^2/2}\;\le \Phi(x) \;\le
\frac{-1}{x}\frac{1}{\sqrt{2\pi}}e^{-x^2/2}.
\]
\end{lemma}

\end{document}